\documentclass[sigconf]{acmart}

\usepackage{booktabs} 

\usepackage{color}
\usepackage{lipsum}  
\usepackage{amsmath}
\usepackage{graphicx}
\usepackage{url}
\usepackage[keeplastbox]{flushend}
\usepackage{amssymb}
\usepackage{subcaption}
\usepackage{multirow}
\usepackage{tikz}
\usepackage{framed}
\usepackage{algorithm}
\usepackage[noend]{algpseudocode}
\usepackage{flushend}

\newcommand{\norm}[1]{\left\lVert#1\right\rVert}
\DeclareMathOperator*{\argmin}{arg\,min}

\copyrightyear{2018} 
\acmYear{2018} 
\setcopyright{acmcopyright}
\acmConference[KDD 2018]{The 24th ACM SIGKDD International Conference on Knowledge Discovery \& Data Mining}{August 19--23, 2018}{London, United Kingdom}
\acmBooktitle{KDD 2018: The 24th ACM SIGKDD International Conference on Knowledge Discovery \& Data Mining, August 19--23, 2018, London, United Kingdom}
\acmPrice{15.00}
\acmDOI{10.1145/3219819.3219952}
\acmISBN{978-1-4503-5552-0/18/08}

\begin{document}
\fancyhead{}

\title{Unlearn What You Have Learned: Adaptive Crowd Teaching \\with Exponentially Decayed Memory Learners}

\author{Yao Zhou}
\orcid{1234-5678-9012}
\affiliation{%
  \institution{Arizona State University}
  \city{Tempe} 
  \state{Arizona} 
  \postcode{85281}
}
\email{yzhou174@asu.edu}

\author{Arun Reddy Nelakurthi}
\affiliation{%
  \institution{Arizona State University}
  \city{Tempe} 
  \state{Arizona} 
  \postcode{85281}
}
\email{arunreddy@asu.edu}

\author{Jingrui He}
\affiliation{%
  \institution{Arizona State University}
  \city{Tempe} 
  \state{Arizona} 
  \postcode{85281}
}
\email{jingrui.he@asu.edu}

\renewcommand{\shortauthors}{Yao Zhou, Arun Reddy Nelakurthi, Jingrui He}

\begin{abstract}
With the increasing demand for large amount of labeled data, crowdsourcing has been used in many large-scale data mining applications. However, most existing works in crowdsourcing mainly focus on label inference and incentive design. In this paper, we address a different problem of adaptive crowd teaching, which is a sub-area of machine teaching in the context of crowdsourcing. Compared with machines, human beings are extremely good at learning a specific target concept (e.g., classifying the images into given categories) and they can also easily transfer the learned concepts into similar learning tasks. Therefore, a more effective way of utilizing crowdsourcing is by supervising the crowd to label in the form of teaching. In order to perform the teaching and expertise estimation simultaneously, we propose an adaptive teaching framework named JEDI to construct the personalized optimal teaching set for the crowdsourcing workers. In JEDI teaching, the teacher assumes that each learner has an exponentially decayed memory. Furthermore, it ensures comprehensiveness in the learning process by carefully balancing teaching diversity and learner's accurate learning in terms of teaching usefulness. Finally, we validate the effectiveness and efficacy of JEDI teaching in comparison with the state-of-the-art techniques on multiple data sets with both synthetic learners and real crowdsourcing workers. 
\end{abstract}

%
%




\maketitle

\section{Introduction}
In many real-world applications, the performance of the learning models usually depends on the quality and the amount of labeled training examples. With the increasing attention on the large-scale data mining problems, the demand for large amount of labeled data also grows at an unprecedented scale. One of the most popular means of collecting the labeled data is through crowdsourcing platforms, such as Amazon Mechanical Turk, Crowdflower, etc. With the help of these crowdsourcing services, where the data is outsourced and labeled by a group of mostly unskilled online workers, the researchers and organizations are able to obtain large amount of label information within a short period of time at a low cost. However, the labels provided by these workers are often of low-quality due to the lack of expertise and lack of incentives, etc. In recent years, several works \cite{DBLP:journals/corr/ZhouLPMS15,DBLP:conf/nips/LiuPI12, DBLP:conf/sdm/ZhouYH17, TAC,M2VW} have been proposed to model and to estimate the expertise of the workers, and these approaches tend to improve the collective labeling quality by downweighting the votes from the weak annotators and trusting the experts. Another branch of crowdsourcing research \cite{DBLP:conf/icml/ShahZP15,DBLP:conf/icml/ShahZ16} focuses on the design of incentives that could motivate the workers to convey their knowledge more accurately by coupling it with a well-designed compensation mechanism. Despite the success of these works, they all omitted one important fact: human beings are extremely good at learning a specific target concept (e.g., classifying the images into given categories) and they can easily transfer the learned concepts into similar learning tasks especially when they have grasped certain prior knowledge regarding the original learning concept. Based on the above insightful observations, it is commonly assumed that a more effective way of utilizing crowdsourcing is by supervising the crowd to label in the form of teaching \cite{DBLP:conf/icml/SinglaBBKK14,DBLP:conf/cvpr/JohnsAB15}.

The crowdsourcing workers usually have a variety of expertise. Therefore, teaching them a certain concept and estimating their labeling abilities at the same time is a challenging problem in general. From the context of teaching, there is an emerging research direction named machine teaching \cite{MT_overview} which is the inverse problem of machine learning. Given the learners, the learning algorithm, and the target concept, machine teaching is concerned with a teacher who wants the learner to learn this target concept as fast as possible. Usually, the main principle of machine teaching is to improve the efficacy of the teachers either by minimizing the teaching effort (i.e., the teaching dimension \cite{DBLP:journals/jmlr/LiuZ16,DBLP:conf/aaai/Zhu15}, which is defined as the cardinality of the optimal teaching set), or by maximizing the converging speed \cite{DBLP:conf/icml/LiuDHTYSRS17} (i.e., the number of the teaching iterations to reach teaching optimum). In this work, we focus on the problem of adaptive crowd teaching, which is a sub-area of machine teaching in the context of crowdsourcing. In crowd teaching, the learners are the crowdsourcing workers and the teacher is the machine that guides the teaching procedure. This is similar to the computer tutoring system, where the teacher teaches by demonstrating the typical examples with answers to the students, and the teacher's goal is to help the students have good performance in similar tasks after tutoring. Very few works \cite{DBLP:conf/cvpr/JohnsAB15,DBLP:conf/icml/SinglaBBKK14} have been conducted to solve this problem, however, none of them have considered the human memory decay during learning, which has been shown to strongly affect real human learner's categorization decisions~\cite{DBLP:conf/nips/PatilZKL14, memoryretrieval, forgettingcurve}.

We propose an adaptive teaching paradigm based on the assumption that the learners have exponentially decayed memories \cite{forgettingcurve}. Within our proposed paradigm, the teacher can gradually construct a personalized optimal teaching sequence for each learner by recommending a teaching example and querying the response from the learner interactively with multiple teaching iterations. Moreover, our teaching strategy ensures the teaching sequence diversity to help the learner develop more comprehensive knowledge on the learning task, and guarantees the teaching sequence usefulness to increase the learner's learning accuracy. To be specific, the main contributions of our work are summarized as follows:
\begin{itemize}
    \item \textbf{Formulation}: We formulate the crowd teaching as a pool-based searching problem which performs teaching and expertise estimation simultaneously. The key idea of this teaching framework is to impose the trade-off between the principles of maximizing teaching usefulness and teaching diversity.
    \item \textbf{Models}: Each learner is assumed to be a gradient descent based model with exponentially decayed memory, and the teacher is formulated to minimize the discrepancy between the learner's current concept and the target concept. We also provide theoretical analyses to study the quality of the teaching examples.
    \item \textbf{Experiments}: We have provided a visualization of our teaching framework on a two-dimensional toy data set, and an exhaustive comparison on one synthetic data set and two real-world data sets using simulated learners. Furthermore, we have conducted a teaching experiment on real human learners and compared our results with state-of-the-art techniques with promising results.
    \item \textbf{Demonstration}: We have built a web-based teaching interface\footnote{A demo of this teaching interface is available at: \href{http://198.11.228.162:9000/memory/index/}{\underline{JEDI-Web-Demo}}. The latest source code is available at: \href{https://github.com/collwe/JEDI-Crowd-Teaching/}{\underline{JEDI-Crowd-Teaching}}.} for real human learners. This interface includes all three modules of teaching: memory length estimation, interactive teaching, and performance evaluation. 
\end{itemize}

The rest of this paper is organized as follows. Section 2 briefly reviews some related work. In Section 3, we formally introduce the model of the learner and the model of the teacher, followed by the discussions of the algorithm and the analysis of the teaching performance. The adaptive teaching using harmonic function is proposed in Section 4. The experimental results are presented in Section 5 and we conclude this paper in Section 6.

\section{Related Work}
In this section, we start by reviewing the research in machine teaching followed with its recent advances. Next, we will review the crowdsourcing works that have some overlap with machine teaching. In the end, we also introduce several other works closely related to human learning and teaching.

\subsection{Machine Teaching}
The inverse problem of machine learning is named as machine teaching, which typically assumes that there is a teacher who knows both the target concept and the learning algorithm used by a learner. Then, the teacher wants to teach the target concept to the learner by constructing an optimal teaching set of examples. One classic definition of this "optimal" is teaching dimension \cite{DBLP:journals/jcss/GoldmanK95} which is referred to as the cardinality of the teaching set. Finding the optimal teaching set which strictly minimizes the teaching dimension is a difficult problem to solve in general. Thus, a relaxed formulation \cite{MT_overview} of machine teaching has been proposed as an optimization problem that minimizes both the teaching risk and the teaching cost. In recent years, there has been a wide range of applications related to machine teaching, e.g., crowdsourcing \cite{DBLP:conf/icml/SinglaBBKK14,DBLP:conf/cvpr/JohnsAB15}, educational tutoring, and data poisoning \cite{DBLP:conf/aaai/MeiZ15, DBLP:conf/icml/XiaoBBFER15}, etc. In the meantime, several theoretical works have studied various aspects of machine teaching such as iterative machine teaching \cite{DBLP:conf/icml/LiuDHTYSRS17}, recursive teaching dimension \cite{DBLP:journals/jmlr/DoliwaFSZ14}, and teaching dimension of linear learners \cite{DBLP:journals/jmlr/LiuZ16}, etc. Our work extends the study of machine teaching into the domain of crowdsourcing, and we studied the crowd teaching problem both theoretically and empirically. 

\subsection{Crowdsourcing}
Crowdsourcing is a special sourcing model in which pieces of micro-tasks are distributed to a pool of online workers. It has become a popular research topic in the recent decades because of its widely commercial and academic adoptions in related areas. One of the fundamental
problems of interest is how to properly guide the online workers and teach them the correct labeling concept given the fact that those hired workers are usually non-experts. Based on the learning and teaching styles \cite{Felder88learningand} that students progress towards concept understanding, human learners can be categorized as either the sequential learners (who learn things in continual steps) or global learners (who learn things in large jumps, holistically). Inspired by this pioneer work, recently, several teaching models have been proposed: the gradient descent model proposed in \cite{DBLP:conf/icml/LiuDHTYSRS17} studied the teaching paradigm for sequential learners and their study conducted on human toddlers has demonstrated the effectiveness of iterative machine teaching; the non-stationary hypothesis transition model proposed in \cite{DBLP:conf/icml/SinglaBBKK14} assumes crowdsourcing workers are global learners and their learned concepts are randomly switched based on observed workers' feedback; the expected error reduction model proposed in \cite{DBLP:conf/cvpr/JohnsAB15} learns to present the most informative teaching images to the students by using an online estimation of their current knowledge. Compared with the former approaches, our work explicitly models the human learner with an exponentially decayed memory which is suitable for the human short-term memory concept learning \cite{HENSON199873}. Meanwhile, our teaching paradigm is an adaptive crowd teaching framework that ensures both the usefulness and the diversity of the teaching examples.

\subsection{Other Related Work}
Besides the existing works on machine teaching and crowdsourcing, the proposed work in this paper is also closely related to many other research subjects such as active learning \cite{DBLP:conf/icml/YanRFD11} and curriculum learning \cite{DBLP:conf/icml/BengioLCW09}. The learner in active learning can query the label of an example from the oracle; however, the teaching example in machine teaching is recommended by the teacher. Curriculum learning, which is inspired by the learning process of humans and animals, suggests an easy-to-complex teaching strategy. The empirical results conducted on human subjects in \cite{DBLP:conf/nips/KhanZM11} have indicated that human teachers tend to follow the curriculum learning principle. In curriculum learning, samples in the teaching sequence are selected merely based on the example difficulty. However, as a comparison, self-paced learning with diversity \cite{DBLP:conf/nips/JiangMYLSH14} which also favors example diversity has shown its superior performance on various learning tasks such as detection and classification. 

\section{The crowd teaching framework}
In this paper, we denote $\mathcal{X} \subset \mathbb{R}^m $ as the $m$-dimensional feature representations of all examples (e.g., images or documents) and $\mathcal{Y}$ as the collection of labels. The teacher has access to a labeled subset $\Phi \subset \mathcal{X} \times \mathcal{Y}$, which is named as the teaching set\footnote{The definition of teaching set in this paper is the same as in \cite{DBLP:conf/icml/SinglaBBKK14,DBLP:conf/cvpr/JohnsAB15}. However, in the concept of teaching dimension \cite{DBLP:journals/jmlr/LiuZ16}, the definition of teaching set is different from ours.} of the teaching task. For binary concept learning, $\textbf{x} \in \mathcal{X}$ is the feature representation of one example, and $y \in \{-1, +1\}$ is its corresponding binary class label. We assume the teacher knows the target concept $\textbf{w}_* \in \mathbb{R}^m$ and the learning model (e.g., logistic regression) of each learner. The teacher wants to teach the target concept to the learner using a personalized teaching set which is constructed by interacting with the learner for multiple teaching iterations. To be specific, each teaching iteration (e.g., the $t$-th iteration) includes the following three major steps:
\begin{itemize}
    \item First, the teacher estimates the current concept $\textbf{w}_{t-1}$ grasped by the learner and recommends a new teaching example ($\textbf{x}_t, y_t$) to the learner.
    \item Next, the teacher will show the recommended teaching example (without revealing its true label $y_t$), and ask the learner to provide its label estimation $\tilde{y}_t$.
    \item At last, the teacher reveals the true label $y_t$ to the learner, and the learner will perform the learning use ($\textbf{x}_t, y_t$).
\end{itemize}
\subsection{Model of the Learner}
To begin with, we assume that the learners to be taught are active learners who are seeking for improvement and aim to become the experts of the given task. Therefore, we do not take the spammers or adversaries into consideration under this teaching setting. 

Now, we formally introduce the model of the learner, whose assets include its \textit{initial concept $\textbf{w}_0$}, \textit{learning loss $\mathcal{L}(\cdot,\cdot)$}, \textit{learning procedure}, and \textit{learning rate $\eta_t$}. After the $t$-th teaching iteration, the learner applies a linear model, i.e., $\textbf{w}_t^T \textbf{x}$, to predict using its learned concept $\textbf{w}_t$. Similar to the learning model proposed in \cite{DBLP:conf/icml/LiuDHTYSRS17}, we also assume that the learner uses a gradient descent learning procedure. However, based on the fact that the real human learner's categorization decisions are guided by a small set of examples retrieved from memory at the time of decision \cite{DBLP:conf/nips/PatilZKL14, memoryretrieval}, and the retrievability of memory is usually approximated with an exponential curve \cite{forgettingcurve}, we further assume that each learner has an exponentially decayed retrievability for the learned concept in terms of the order of the teaching examples, i.e.,:
\begin{equation}
    \begin{split}
        \textbf{v}_1 &= \beta \textbf{v}_0 + \frac{\partial \mathcal{L} (\textbf{w}_0^T \textbf{x}_1, y_1)}{\partial \textbf{w}_0} \\
        \textbf{v}_2 &= \beta \textbf{v}_1 + \frac{\partial \mathcal{L} (\textbf{w}_1^T \textbf{x}_2, y_2)}{\partial \textbf{w}_1} \\
            &= \beta^2 \textbf{v}_0 + \Big[ \beta \frac{\partial \mathcal{L} (\textbf{w}_0^T \textbf{x}_1, y_1)}{\partial \textbf{w}_0} + \frac{\partial \mathcal{L} (\textbf{w}_1^T \textbf{x}_2, y_2)}{\partial \textbf{w}_1} \Big] \\
        &\hdots \\
        \textbf{v}_t &= \beta^t \textbf{v}_0 + \sum^t_{s=1}{\beta^{t-s} \frac{\partial \mathcal{L}(\textbf{w}_{s-1}^T\textbf{x}_{s}, y_s)}{\partial \textbf{w}_{s-1}}}
    \end{split}
\end{equation}
where $\beta \in (0,1)$ is the personalized memory decay rate. Various learners can have different memory lengths, and this personalized memory length is parameterized by $\beta$. The learners with large $\beta$ can actually retrieve more information from their memory. The concept momentum $\textbf{v}_t$ is defined as the linear combination of its previous concept momentum $\textbf{v}_{t-1}$ and the gradient of the learner's loss $\frac{\partial \mathcal{L}(\textbf{w}_{t-1}^T \textbf{x}_t, y_t)}{\partial \textbf{w}_{t-1}}$. The initial momentum $\textbf{v}_0$ is usually set to \textbf{0} in practice. With the properly chosen learning rate $\eta_t$, the learner uses the gradient descent learning procedure to improve their concept in an iterative way:
\begin{equation} \label{LearnerUpdate}
    \textbf{w}_{t} \leftarrow \textbf{w}_{t-1} - \eta_t \textbf{v}_t 
\end{equation}
Similar to stochastic gradient descent (SGD) with momentum, the learner will update his/her concept $\textbf{w}_t$ towards the target concept $\textbf{w}_*$ along the direction of the negative concept momentum  $-\textbf{v}_t$, which is the linear combination of the negative gradients of the learning losses with exponentially decayed weights. Intuitively, the concept learned by a human learner depends on a sequence of teaching examples. The latest example will contribute more (has larger weights) towards learning than the earlier ones. 
\begin{figure}[!t]
  \centering
    \includegraphics[scale = 0.33]{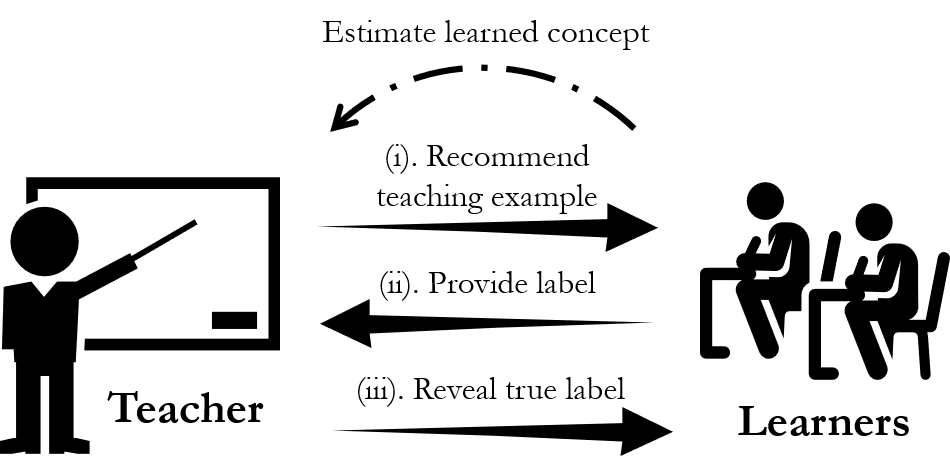}
  \caption{Illustration of JEDI (one teaching iteration)} 
  \label{N/A}
\end{figure}
\subsection{Model of the Teacher}
Initially, we assume that the teacher has access to the learner's current concept $\textbf{w}_t$, learning loss, learning procedure, etc., and the teacher intends to guide the learner towards the target concept $\textbf{w}_*$. Notice that in real-world teaching, the teacher generally does not have direct access to a learner's current concept. The alternative of estimating a learner's concept will be introduced in Section 4. Thus, the objective of teaching is proposed as follows:
\begin{equation}
    \begin{split}
        \min || \textbf{w}_t - \textbf{w}_* ||_2^2  
    \end{split}
\end{equation}

This objective is designed to minimize the discrepancy between the target concept $\textbf{w}_*$ and the learner's current concept $\textbf{w}_t$ after $t$ rounds of teaching. The objective can be decomposed into three parts by substituting Eq. (\ref{LearnerUpdate}) into it:
\begin{equation}
    \begin{split}
        \mathcal{O}(\textbf{x}_t,y_t)  =&\norm{ \textbf{w}_t - \textbf{w}_* }_2^2  \\
        = & \norm{ \textbf{w}_{t-1} - \textbf{w}_* }_2^2 + \eta_t^2 \underbrace{\norm{ \sum_{s=1}^t  \beta^{t-s} \frac{\partial \mathcal{L}(\textbf{w}_{s-1}^T \textbf{x}_s, y_s)}{\partial \textbf{w}_{s-1}} }_2^2}_\text{$T_1$: Diversity of the teaching sequence} \\
         - & 2\eta_t \underbrace{\bigg< \textbf{w}_{t-1} - \textbf{w}_*, \sum^{t}_{s=1} \beta^{t-s} \frac{\partial \mathcal{L}(\textbf{w}_{s-1}^T \textbf{x}_s, y_s)}{\partial \textbf{w}_{s-1}} \bigg>}_\text{$T_2$: Usefulness of the teaching sequence}
    \end{split}
\end{equation}
The first part is the discrepancy between $\textbf{w}_*$ and learner's previous concept $\textbf{w}_{t-1}$, and the second part $T_1$ essentially measures the diversity of the teaching sequence. The third part $T_2$ measures the usefulness of the teaching sequence and the intuitive explanations of $T_1$ and $T_2$ will be clear later. 

Meanwhile, we assume that the teacher has an infinite memory of the teaching sequence of examples $\mathcal{D}_t = \{ (\textbf{x}_1, y_1), \,\hdots, (\textbf{x}_t, y_t)\}$, as well as the corresponding estimate of the concept sequence from the learner $\mathcal{W}_t = \{ \textbf{w}_0, \textbf{w}_1, \,\hdots, \textbf{w}_t\}$. 

\vspace{3mm}
\noindent \textbf{\large Diversity of the teaching sequence}.
\vspace{2mm} \\
In order to simplify $T_1$, we further decompose it into two intermediate terms:
\begin{equation}
    \begin{split}
        T_1 &= \sum^t_{s=1} \beta^{2(t-s)} \norm{\frac{\partial \mathcal{L}(\textbf{w}_{s-1}^T \textbf{x}_s, y_s)}{\partial \textbf{w}_{s-1}}}^2_2\\
        &+ \sum^t_{s=1}\sum^t_{r\neq s} \beta^{2t-s-r} \bigg< \frac{\partial \mathcal{L}(\textbf{w}_{s-1}^T \textbf{x}_s, y_s)}{\partial \textbf{w}_{s-1}}, \frac{\partial \mathcal{L}(\textbf{w}_{r-1}^T \textbf{x}_r, y_r)}{\partial \textbf{w}_{r-1}} \bigg>
    \end{split}
\end{equation}

The selection of the learning loss can be flexible. For the task of teaching a classification concept, we utilize the logistic loss, which is convex and smooth, to illustrate the key idea, i.e., $\mathrm{log}\big(1+\mathrm{exp}(-y \textbf{w}^T \textbf{x}) \big)$, although the proposed framework can be extended to other loss functions. Easily, we can have the gradient norm of each teaching example as $\Big( \frac{-y}{1+\mathrm{exp}( y\textbf{w}^T\textbf{x} )} \Big)^2 \norm{\textbf{x}}^2_2$, which has the interpretation of \textit{example difficulty} when all the example feature $\textbf{x}$ lies on a hypersphere (e.g., L2-normalized bag-of-words features in document classification). In that case, $\norm{\textbf{x}}_2 = 1$ and the first term of $T_1$ becomes the sum of squares of the probability of incorrect predictions. 

\begin{figure}[t]
    \centering
    \fbox{\includegraphics[scale = 0.3]{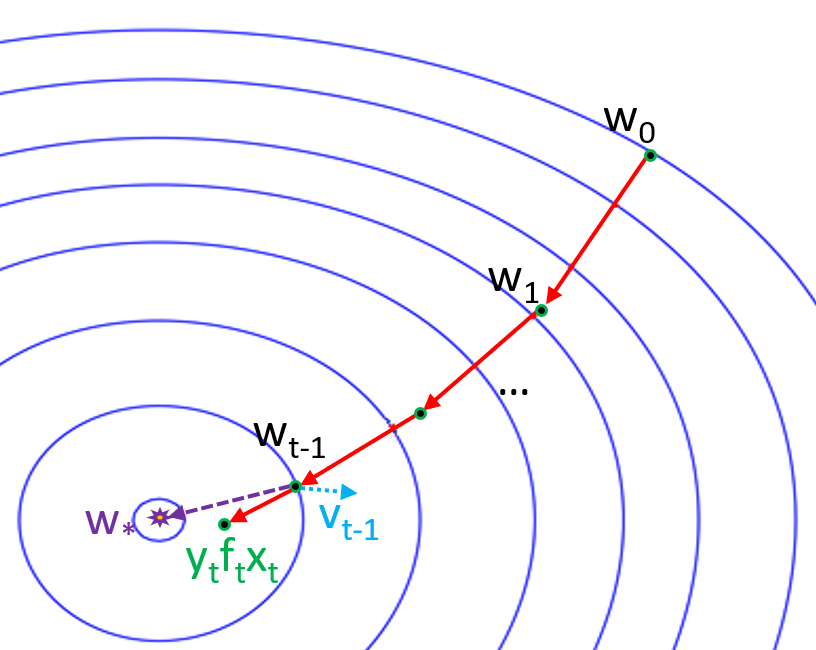}}
    \caption{Trade-off between diversity and usefulness}
    \label{tradeoff}
\end{figure}

Our goal of teaching is to recommend the next teaching example $(\textbf{x}_t, y_t)$, therefore, these observed gradients (with indices $s=1, \hdots, t-1$) are not relevant in this teaching optimization sub-problem of minimizing $T_1$. If we substitute the gradient of logistic loss into the objective, it is straightforward to get the following equivalent optimization sub-problem:
\begin{equation}
    \begin{split}
        &\min_{(\textbf{x}_t, y_t)} \; T_1 \\
        \Leftrightarrow  &\min_{(\textbf{x}_t, y_t) } \bigg( \frac{y_t}{1+\mathrm{exp}(y_t\textbf{w}_{t-1}^T \textbf{x}_t)} \bigg)^2 \norm{\textbf{x}_t}_2^2\\
        &+ \sum^{t-1}_{s=1} 2\beta^{t-s} \bigg[ \frac{y_s y_t}{\Big(1+\mathrm{exp}(y_s\textbf{w}_{s-1}^T \textbf{x}_s)\Big) \Big(1+\mathrm{exp}(y_t\textbf{w}_{t-1}^T \textbf{x}_t)\Big)} \bigg] \textbf{x}^T_s \textbf{x}_t
    \end{split}
\end{equation}

\vspace{3mm}
\noindent \textbf{\large Usefulness of the teaching sequence}.
\vspace{2mm} \\
$T_2$ part in the objective serves as the measurement of the usefulness of the whole teaching sequence. Specifically, it is the weighted sum of all inner products between $\textbf{w}_{t-1} - \textbf{w}_*$ and the gradients of the teaching sequence examples. It means that the entire teaching sequence $\mathcal{D}_t$ will contribute to maximizing the convergence of the teaching. The larger value the inner product has, the more useful this teaching sequence is. However, similar to the $T_1$ minimization sub-problem, only $(\textbf{x}_t, y_t)$ relevant terms matter for the purpose of maximizing $T_2$:
\begin{equation}
    \begin{split}
        &\max_{(\textbf{x}_t, y_t) } \; T_2\\ 
        \Leftrightarrow &\max_{(\textbf{x}_t, y_t) }  \sum^{t-1}_{s=1} \beta^{t-s} \bigg< \textbf{w}_{t-1} - \textbf{w}_*, \frac{\partial \mathcal{L}(\textbf{w}_{s-1}^T \textbf{x}_s, y_s)}{\partial \textbf{w}_{s-1}} \bigg> \\
        &+ \bigg< \textbf{w}_{t-1} - \textbf{w}_*, \frac{\partial \mathcal{L}(\textbf{w}_{t-1}^T \textbf{x}_t, y_t)}{\partial \textbf{w}_{t-1}} \bigg> \\
        \Leftrightarrow &\max_{(\textbf{x}_t, y_t) } \bigg< \textbf{w}_{t-1} - \textbf{w}_*, \frac{-y_t \textbf{x}_t}{1+\mathrm{exp}(y_t\textbf{w}_{t-1}^T \textbf{x}_t)} \bigg>
    \end{split}
\end{equation}

\vspace{3mm}
\noindent \textbf{\large Trade-off between diversity and usefulness}.
\vspace{2mm} \\
For simplicity, we denote $f_t := \frac{1}{1+\mathrm{exp}(y_t \textbf{w}_{t-1}^T \textbf{x}_t)}$ and $f_s := \frac{1}{1+\mathrm{exp}(y_s \textbf{w}_{s-1}^T \textbf{x}_s)}$, where $s = 1,\hdots, t-1$. Then, the overall teaching problem becomes:

\begin{equation} \label{overallOBJ}
    \begin{split}
        &\min_{(\textbf{x}_t, y_t)} \mathcal{O}(\textbf{x}_t,y_t) \\
        \Leftrightarrow &\min_{(\textbf{x}_t, y_t)} \eta_t^2 \bigg[ \big< y_t f_t \textbf{x}_t, y_t f_t \textbf{x}_t \big> + 2\sum^{t-1}_{s=1} \beta^{t-s} \big< y_s f_s \textbf{x}_s, y_t f_t \textbf{x}_t \big>  \bigg] \\
        &-2\eta_t \big< \textbf{w}_{t-1}-\textbf{w}_*, -y_t f_t \textbf{x}_t \big> \\
        \Leftrightarrow &\min_{(\textbf{x}_t, y_t) } \eta_t^2 \norm{y_t f_t \textbf{x}_t - \textbf{v}_{t-1}}^2_2 - 2\eta_t \big< \textbf{w}_* -\textbf{w}_{t-1}, y_t f_t \textbf{x}_t \big>
    \end{split}
\end{equation}


Prob. (\ref{overallOBJ}) aims to maximize the teaching diversity ($T_1$ part) and the teaching usefulness ($T_2$ part) at the same time. As illustrated in Figure \ref{tradeoff}, the teacher prefers the negative gradient $y_t f_t \textbf{x}_t$ of next teaching example is similar to the concept momentum $\textbf{v}_{t-1}$ and has large correlation with the target learning direction $\textbf{w}_* - \textbf{w}_{t-1}$. The learning rate $\eta_t $ is usually set to a small value ($\eta_t < 1$) in optimization. Therefore, the teaching usefulness ($T_2$ part) dominates the teaching process. It is straightforward to see that when the $\beta = 0$, our objective is directly reduced to the no memory teaching framework proposed in \cite{DBLP:conf/icml/LiuDHTYSRS17}.

In order to solve this teaching problem, we denote $\textbf{a}_t := y_t f_t \textbf{x}_t$. Then, the teaching objective can be further simplified as:
    \begin{equation}
        \begin{split}
            \mathcal{O}(\textbf{x}_t, y_t) &= \eta_t^2 \bigg[ \norm{\textbf{a}_t}_2^2 - 2\bigg< \textbf{a}_t, \textbf{v}_{t-1} + \frac{\textbf{w}_* - \textbf{w}_{t-1}}{\eta_t} \bigg> \bigg]\\
            &= \eta_t^2 \norm{\textbf{a}_t - \bigg( \textbf{v}_{t-1} + \frac{\textbf{w}_* - \textbf{w}_{t-1}}{\eta_t} \bigg)}_2^2 \\
            &- \eta_t^2 \norm{\textbf{v}_{t-1} + \frac{\textbf{w}_* - \textbf{w}_{t-1}}{\eta_t} }_2^2\\
        \end{split}
    \end{equation}
The concept momentum $\textbf{v}_{t-1}$ is the weighted sum of gradients of the teaching sequence $\mathcal{D}_{t-1}$ using exponentially decayed weights. 

Based on this new objective, we propose the teaching algorithm \textbf{JEDI} (Ad\textbf{J}ustable \textbf{E}xponential \textbf{D}ecay Memory \textbf{I}nteractive Crowd Teaching). The JEDI teaching algorithm with omniscient teacher (having access to learner's concept sequence $\mathcal{W}_t$) is shown in \textbf{Algorithm} \ref{firstAlg}. It is given the learner's memory decay rate, initial concept, target concept, learning rate, teaching set as input, and will output the personalized teaching sequence. JEDI works as follows. We first initialize the teaching iterator $t = 1$ and initial momentum $\textbf{v}_0 = \textbf{0}$. Then, in each iteration, the teacher searches through the teaching set $\Phi$ and finds the example ($\textbf{x}_t, y_t$) that minimizes the objective function in Eqn. (\ref{JEDI_Omni_alg}), where $f=\frac{1}{1+\mathrm{exp}(y \textbf{w}_{t-1}^T \textbf{x})}$ is the probability of incorrect prediction of example ($\textbf{x},y$) in the teaching set $\Phi$. Next, the learner performs the labeling on $\textbf{x}_t$ using its current concept $\textbf{w}_{t-1}$. Then, the label $y_t$ is revealed by the teacher and the learner performs learning using Eqn. (\ref{LearnerUpdate}). This interactive teaching will continue until the stopping criteria is satisfied.

In exponential weighted average, the number of examples being used is usually approximated \cite{DBLP:journals/corr/KingmaB14} as $\frac{1}{1-\beta}$. Therefore, we can assume that there exists a memory window size (i.e,. how many examples or corresponding concept gradients the learner can memorize) for each learner, and it can be approximated as $\frac{1}{1-\beta}$. In the following, we use the two-example teaching scenario (e.g., memory decay rate is as low as $\beta \approx 0.5$, or it is the second iteration of teaching $t = 2$) as a running example, where the learner can memorize a teaching sequence of size 2. Notice that the analysis and conclusions can be extended to other values of $\beta$ as well. In the two-example teaching scenario, the trade-off between diversity and usefulness will lead to further insights with the help of the following definitions and theorem.
\begin{algorithm}[t]
\caption{JEDI with omniscient teacher}
\label{firstAlg}
\begin{algorithmic}[1]
\State \textbf{Input:} Learner's memory decay rate $\beta$, initial concept $\textbf{w}_0$, target concept $\textbf{w}_*$, initial learning rate $\eta_0$, teaching set $\Phi$, MaxIter.
\State \textbf{Initialization:} 
\begin{equation*}
    \begin{split}
        \textbf{v}_0 &\leftarrow \textbf{0} \\
        t &\leftarrow 1 
    \end{split}
\end{equation*}
\State  \textbf{Repeat:}
\State  \quad (i). Among all examples ($\textbf{x},y$) in teaching set $\Phi$ and their probabilities of incorrect prediction $f$, the teacher recommends example ($\textbf{x}_t, y_t$) to the learner by solving:
\begin{equation} \label{JEDI_Omni_alg}
    \hspace{5mm}
    (\textbf{x}_t, y_t) = \argmin_{(\textbf{x}, y) \in \Phi } \norm{y f \textbf{x} - \bigg( \textbf{v}_{t-1} + \frac{\textbf{w}_* - \textbf{w}_{t-1}}{\eta_t} \bigg)}_2^2
\end{equation}
\State \quad (ii). Learner performs the labeling.
\State \quad (iii). Learner performs learning after teacher reveals $y_t$.
\State \quad (iv). t $\leftarrow$ t + 1
\State  \textbf{Until} converged \textbf{or} $t >$ MaxIter 
\State \textbf{Output:} The teaching sequence $\mathcal{D}_t$
\end{algorithmic}
\end{algorithm}

\begin{definition}
    Given the previous teaching example ($\textbf{x}_{t-1}, y_{t-1}$), if the teacher recommends a new teaching example ($\textbf{x}_{t}, y_{t}$) which has different label $y_{t} \neq y_{t-1}$, this teaching action is named \textbf{Exploration}. If the new teaching example has the same label $y_{t} = y_{t-1}$, this teaching action is named \textbf{Exploitation}.
\end{definition}
\begin{definition}
    Given the previous teaching example ($\textbf{x}_{t-1}, y_{t-1}$), its negative gradient $y_{t-1} f_{t-1} \textbf{x}_{t-1}$ and its optimal teaching direction $\textbf{w}_* - \textbf{w}_{t-1}$ has an angle $\theta \in [0, \pi]$. Then, this teaching example is \textbf{not useful} towards teaching optimal $\textbf{w}_*$ if the angle satisfies $\theta \geq \frac{\pi}{2}$.
\end{definition}

\begin{theorem} \label{EE}
    (\textbf{Exploration vs. Exploitation}) For two-example teaching, if the previous teaching example ($\textbf{x}_{t-1}, y_{t-1}$) is not useful towards teaching optimal, the teacher will recommend large diversity teaching example ($\textbf{x}_{t}, y_{t}$) for exploitation, i.e., $y_{t}= y_{t-1}$, or recommend highly similar teaching example ($\textbf{x}_{t}, y_{t}$) for exploration, i.e., $y_{t}\neq y_{t-1}$.
\end{theorem}
\begin{proof}
    Let $\textbf{a}_{t-1} := \beta y_{t-1} f_{t-1} \textbf{x}_{t-1}$, then the teaching objective becomes:
    \begin{equation*}
        \begin{split}
            \mathcal{O}(\textbf{x}_t, y_t) &= \eta_t^2 \norm{\textbf{a}_t + \bigg( \textbf{a}_{t-1} - \frac{\textbf{w}_* - \textbf{w}_{t-1}}{\eta_t} \bigg)}_2^2 \\
            &- \eta_t^2 \norm{\textbf{a}_{t-1} - \frac{\textbf{w}_* - \textbf{w}_{t-1}}{\eta_t} }_2^2\\
        \end{split}
    \end{equation*}
    
    The minimum of the objective is guaranteed to be reached when the next teaching example $\textbf{x}_t$ is selected as follows:
    \begin{equation*}
        \textbf{x}_t = -\beta \frac{y_{t-1}}{y_t} \frac{f_{t-1}}{f_t} \textbf{x}_{t-1} + \frac{1}{y_t f_t} \frac{\textbf{w}_* - \textbf{w}_{t-1}}{\eta_t} 
    \end{equation*}
    
    Using $\textbf{a}_{t-1}$ as the reference, the optimal teaching direction vector can be decomposed as $\textbf{w}_* - \textbf{w}_{t-1} = (\textbf{w}_* - \textbf{w}_{t-1})_{\parallel} + (\textbf{w}_* - \textbf{w}_{t-1})_{\bot}$ in $\textbf{a}_{t-1}$'s parallel direction and perpendicular direction. If the previous teaching example is not useful (i.e., $\theta \geq \frac{\pi}{2}$), without loss of generality, we can assume $(\textbf{w}_* - \textbf{w}_{t-1})_{\parallel} = \alpha \textbf{a}_{t-1}$, where $\alpha \leq 0$ is obviously satisfied. Then, we have:
    \begin{align*}
        \textbf{x}_t = -\gamma_+ \frac{f_{t-1}}{f_t} \textbf{x}_{t-1} &+ \xi_t && (\Leftarrow \mathrm{Exploitation} )\\
        \textbf{x}_t = \gamma_+ \frac{f_{t-1}}{f_t} \textbf{x}_{t-1}  &+ \xi_t &&  (\Leftarrow \mathrm{Exploration})
    \end{align*}
    where $\gamma_+ = (1-\frac{\alpha}{\eta_t})\beta$ is a positive scalar and $\xi_t = \frac{1}{y_t f_t} \frac{(\textbf{w}_* - \textbf{w}_{t-1})_{\bot}}{\eta_t}$ is the teaching perturbation. If the previous teaching example $\textbf{x}_{t-1}$ is not useful, then the teacher will prefer the next teaching example $\textbf{x}_t$ to be very different from the previous one for exploitation (intra-class teaching) or to be similar with the previous one for exploration (inter-class teaching).
\end{proof}

The teaching action choice between exploration and exploitation is very clear especially when the previous teaching example is most useless (i.e., $\theta = \pi$), under which scenario the recommended teaching example has zero teaching perturbation $(\textbf{w}_* - \textbf{w}_{t-1})_{\bot} = 0$. The magnitude of the teaching perturbation is positively correlated with the usefulness of the previous teaching example since $(\textbf{w}_* - \textbf{w}_{t-1})_{\bot} \propto \mathrm{sin}(\theta)$ and $\theta \geq \frac{\pi}{2}$. Therefore, if the previous teaching example is less useful ($\theta$ becomes larger), the perturbation will become smaller, and the teacher has less uncertainty to decide whether the next teaching recommendation should be an exploitation action or an exploration action.
\begin{proposition}
    For the examples that live on a hypersphere, if the previous teaching example ($\textbf{x}_{t-1}, y_{t-1}$) is most useless ($\theta = \pi$) towards teaching optimal and the learning rate satisfies $\eta_t \geq \frac{\alpha \beta}{\beta - 1}$, then the teacher recommended example ($\textbf{x}_t, y_t$) is guaranteed to have better labeling quality than ($\textbf{x}_{t-1}, y_{t-1}$), i.e., the learner can correctly label example $\textbf{x}_t$ with higher probability than labeling example $\textbf{x}_{t-1}$. 
\end{proposition}
\begin{proof}
    We have $f_t = (\frac{\alpha}{\eta_t}-1)\beta \frac{y_{t-1}}{y_t} \frac{<\textbf{x}_{t-1}, \textbf{x}_t>}{<\textbf{x}_t, \textbf{x}_t>} f_{t-1}$ from Theorem \ref{EE}. For hyperspherical feature space, $\big| \frac{<\textbf{x}_{t-1}, \textbf{x}_t>}{<\textbf{x}_t, \textbf{x}_t>} \big| \leq 1$ and no matter if the teaching action is exploration or exploitation, the coefficient of $f_{t-1}$ is always smaller than 1. Therefore, $f_t$ (probability of incorrectly labeling $\textbf{x}_t$) is smaller than $f_{t-1}$ (probability of incorrectly labeling $\textbf{x}_{t-1}$).
\end{proof}

For the teaching scenarios with multiple teaching examples (e.g., $\beta$ is large), the above theoretical analyses are also applicable by treating the previous teaching sequence $\mathcal{D}_{t-1}$ as one pseudo teaching example with its decayed negative gradient as $\textbf{v}_{t-1}$. 

\section{Adaptively teaching the human learners}
In this section, we first discuss the challenges for teaching the real human learners. Then, we present the methodology which can estimate the human learner's current concept using the harmonic function. In the end, we formally present the algorithm JEDI teaching with harmonic function estimation.

\subsection{Teaching in the Real World}
\vspace{2mm}
\noindent \textbf{All examples help teaching}. After the teacher reveals the true label of the recommended teaching example, the human learner can improve the concept learning either by verify the correctness of his/her labels or by gaining information from the mistakes he/she made.

\vspace{2mm}
\noindent \textbf{Repeated teaching examples}. Memories are so volatile that human learners have to be provided with repeated examples to strengthen the learned concept. Due to this reason, the teaching sequence $\mathcal{D}_t$ selected from the teaching set $\Phi$ should have repeated examples especially when these examples are incorrectly labeled or the learner's memory window size is small.

\vspace{2mm}
\noindent \textbf{Pool-based teaching}. Similar to the pool-based active learning, in many real-world teaching tasks, the synthetically generated teaching examples that meet the global optimum of JEDI objective are not valid real-world examples (e.g., images, documents). Thus, a pool-based search is a more realistic alternative. In other words, the JEDI teacher will search for the best teaching examples in the teaching set $\Phi$ instead of the whole feature and label space.

\vspace{2mm}
\noindent \textbf{Teacher has no access to learner's concept}. To address this challenge, notice that by utilizing the first-order convexity of the learning loss, we can have:
\begin{equation}
    \bigg< \textbf{w}_{t-1} - \textbf{w}_*, \frac{\partial \mathcal{L}(\textbf{w}_{t-1}^T \textbf{x}_t, y_t)}{\partial \textbf{w}_{t-1}} \bigg> \geq \mathcal{L}(\textbf{w}_{t-1}^T \textbf{x}_t, y_t) - \mathcal{L}(\textbf{w}_*^T \textbf{x}_t, y_t)
\end{equation}
Then, minimizing $T_2$ can be relaxed to the problem of optimizing its lower bound. This relaxation enables the teacher to query the learner's prediction $\mathrm{sign}(\textbf{w}^T \textbf{x})$ instead of requiring access to his/her concept $\textbf{w}$ directly (which is impossible for real human learners). The effectiveness of this relaxation depends on the tightness of the lower bound. Therefore, the smaller $\norm{\textbf{w}_{t-1} - \textbf{w}_*}$ is, the tighter the bound is. In other words, this relaxed problem is gradually becoming a reliable approximation of the original problem with more and more teaching iterations.

\subsection{Concept Estimation using Harmonic Function} \label{harmonicSection}
In the teaching phase, for every observed teaching example (with indices $s = 1,\hdots, t-1$), the teacher has access to the features $\textbf{x}_s$ and the learner provided label $\tilde{y}_s$. However, the teacher still needs the learner's probability of incorrect prediction $f = \frac{1}{1+\mathrm{exp}(y \textbf{w}_{t-1}^T \textbf{x})}$ on every example ($\textbf{x},y$) in $\Phi$ to start teaching. One naive way of estimating $f$ is by using learner provided labels to train a supervised classification model, and predict the unlabeled ones with this classifier to get $f$. However, due to the limited number of labeled examples, a semi-supervised model \cite{DBLP:conf/icml/ZhuGL03,DBLP:conf/nips/ZhouBLWS03} should be more effective than supervised models. One alternative to estimating $f$ is by using graph-based semi-supervised learning method proposed in~\cite{DBLP:conf/icml/ZhuGL03}. Given the teaching sequence $\mathcal{D}_{t-1}$, for every unlabeled example, we can estimate its probability of labels using semi-supervised Gaussian random fields and harmonic functions:
\begin{equation} \label{harmonic}
    F_u = (D_{uu} - A_{uu})^{-1}A_{ul}F_l
\end{equation}
In the above formulation, $A$ is the affinity matrix of all examples and $D$ is a diagonal matrix (with $D_{ii} = \sum_{j=1} A_{ij}$). Matrix $A$ can be reordered and split into four blocks as:
$ A = \left[
\begin{array}{cc}
A_{ll} & A_{lu} \\
A_{ul} & A_{uu}
\end{array}
\right]$ and similar block split operation is applied on $D$ as well. $F_l \in \{0,1\}^{|\mathcal{D}_{t-1}| \times 2}$ is the label matrix associated with learner provided labels, where each element is set to 1 if the corresponding label has been provided by the learner and 0 otherwise. Following this convention, the affinity matrix can be constructed as follow:
\begin{equation}
    A_{ij} = \mathrm{exp}\bigg( -\sum_{d=1}^m \frac{(\textbf{x}_{id} - \textbf{x}_{jd})^2 }{\sigma_d^2} \bigg)
\end{equation}
It should be noticed that the teaching examples could be repeatedly recommended by the JEDI teacher, and this is different from the crowd teaching model of \cite{DBLP:conf/cvpr/JohnsAB15}, which also uses the harmonic function but only allows each example to be recommended once. Therefore, before applying the harmonic solution, we only keep the unique examples that have the latest labels provided by the learner in the teaching sequence. Meanwhile, in order to guarantee all examples in the teaching set $\Phi$ could be recommended for next round of teaching, affinity matrix $A$ are padded using extra nodes and edges constructed from the existing teaching sequence $\mathcal{D}_{t-1}$. After applying the harmonic solution, the labeling probability estimation of every example $\textbf{x}$ in $\Phi$ corresponds to a row (whose entries are $p$ and $1-p$) of matrix $F_u$:
\begin{equation}
    \begin{split}
        P(y=1|\textbf{x}, \mathcal{D}_{t-1}) &= p = \frac{1}{1+\mathrm{exp}(-\textbf{w}^T \textbf{x})}\\
        P(y=-1|\textbf{x}, \mathcal{D}_{t-1}) &= 1-p = \frac{1}{1+\mathrm{exp}(\textbf{w}^T \textbf{x})}
    \end{split}
\end{equation}
To calculate concept momentum $\textbf{v}_{t-1}$ in $T_1$, which utilizes the probability of incorrect prediction $f_{s}$ of teaching example ($\textbf{x}_s,y_s$) where $s=1,\hdots, t-1$, the estimated labeling probabilities are used together with the teacher revealed ground truth label $y_s$. They are calculated as:
\begin{align} \label{estimateT1}
    f_s := \frac{1}{1+\mathrm{exp}(y_s \textbf{w}_{s-1}^T \textbf{x}_s)} = \big( 1-p_s \big)^{\frac{y_s+1}{2}} p_s^{\frac{1-y_s}{2}}
\end{align}
\noindent where $p_s$ is the harmonic probability estimate of teaching example $\textbf{x}_s$. 
Similarly, in order to calculate $T_2$ term, the estimated labeling probabilities are jointly used with ground truth label $y_t$ as:
\begin{align} \label{estimateT2}
    \hspace{8mm} \frac{1}{f_{-t}} := 1+\mathrm{exp}(-y_t \textbf{w}^T_{t-1}\textbf{x}_t) = \Big( \frac{1}{p_t} \Big)^{\frac{y_t+1}{2}} \Big( \frac{1}{1-p_t} \Big)^{\frac{1-y_t}{2}}
\end{align}
\noindent where $p_t$ is the harmonic probability estimate of teaching example $\textbf{x}_t$. 

\begin{algorithm}[t]
\caption{JEDI with harmonic function estimation}
\label{secondAlg}
\begin{algorithmic}[1]
\State \textbf{Input:} Learner's memory decay rate $\beta$, target concept $\textbf{w}_*$, initial learning rate $\eta_0$, teaching set $\Phi$, affinity matrix A, diagonal matrix D, MaxIter.
\State \textbf{Initialization:} 
\begin{equation*}
    \begin{split}
        \textbf{v}_0 &\leftarrow \textbf{0}\\
        t &\leftarrow 1 
    \end{split}
\end{equation*}

\State  \textbf{Repeat:}
\State  \quad (i). Teacher estimates $F_u$ using Eq. (\ref{harmonic}) and calculates $f_s$ and $\frac{1}{f_{-t}}$ using Eq. (\ref{estimateT1}) and Eq. (\ref{estimateT2}).
\State  \quad (ii). Teacher recommends example ($\textbf{x}_t, y_t$) to the learner:
\begin{equation} \label{JEDIharmonic}
    \begin{split}
        \hspace{5mm}
        (\textbf{x}_t,y_t) = \argmin_{(\textbf{x}, y) \in \Phi } \eta_t^2 \norm{y f \textbf{x} - \textbf{v}_{t-1}}^2_2 - 2\eta_t \mathrm{log} \frac{1+\mathrm{exp}(-y \textbf{w}_{t-1}^T \textbf{x})}{1+\mathrm{exp}(-y \textbf{w}_{*}^T \textbf{x})}
    \end{split}
\end{equation}
\State \quad (iii). Learner performs the labeling and then teacher updates $A, D$, and $F_l$.
\State \quad (iv). Learner performs learning after teacher reveals $y_t$.
\State \quad (v). t $\leftarrow$ t + 1
\State  \textbf{Until} $t >$ MaxIter 
\State \textbf{Output:} The teaching sequence $\mathcal{D}_t$
\end{algorithmic}
\end{algorithm}

\begin{figure*}[!t]
\begin{tabular}{cccc}
\begin{subfigure}{0.25\textwidth} \hspace{-3mm} \includegraphics[width=1\columnwidth]{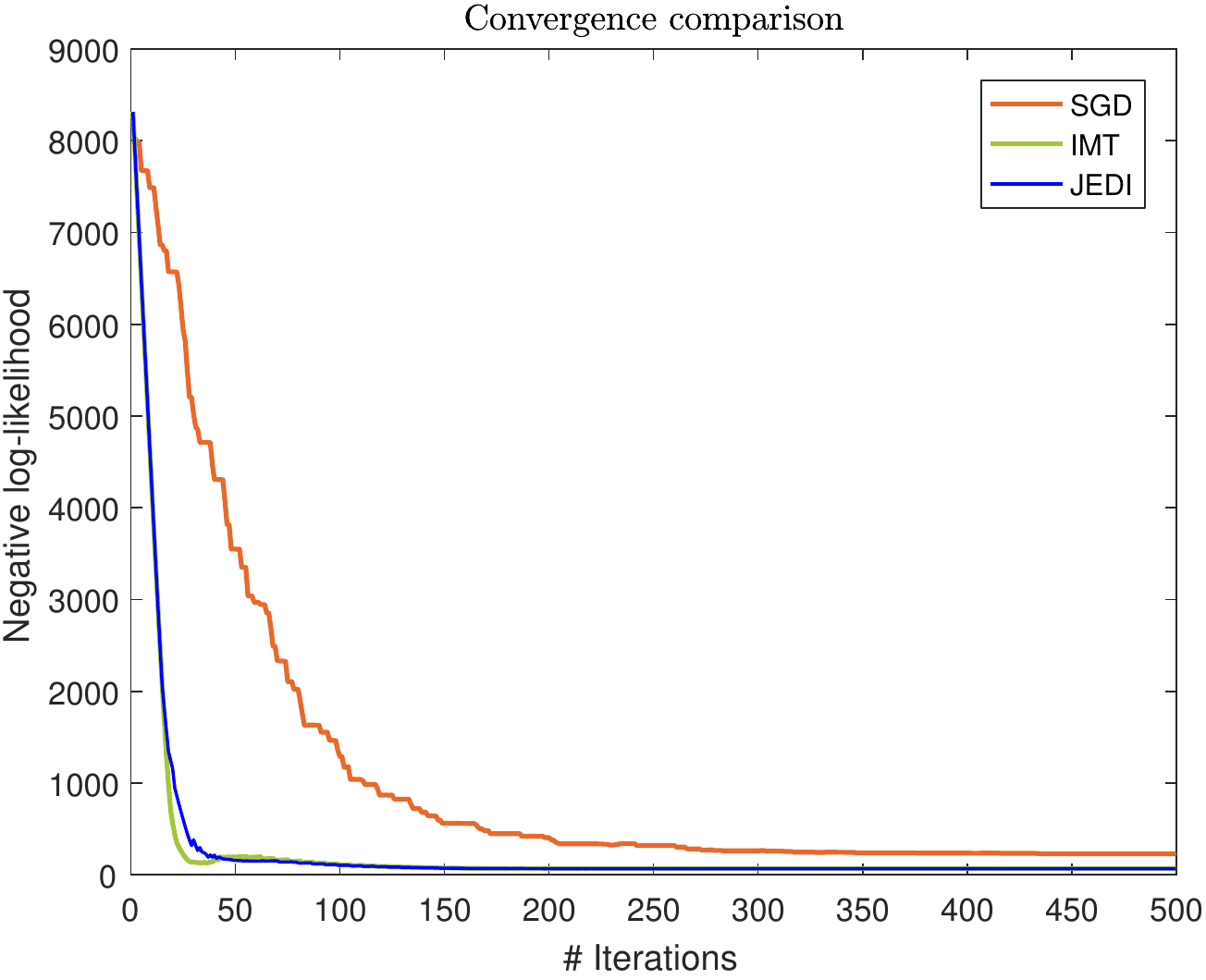}\end{subfigure}&
\begin{subfigure}{0.25\textwidth} \hspace{-6mm} \includegraphics[width=1\columnwidth]{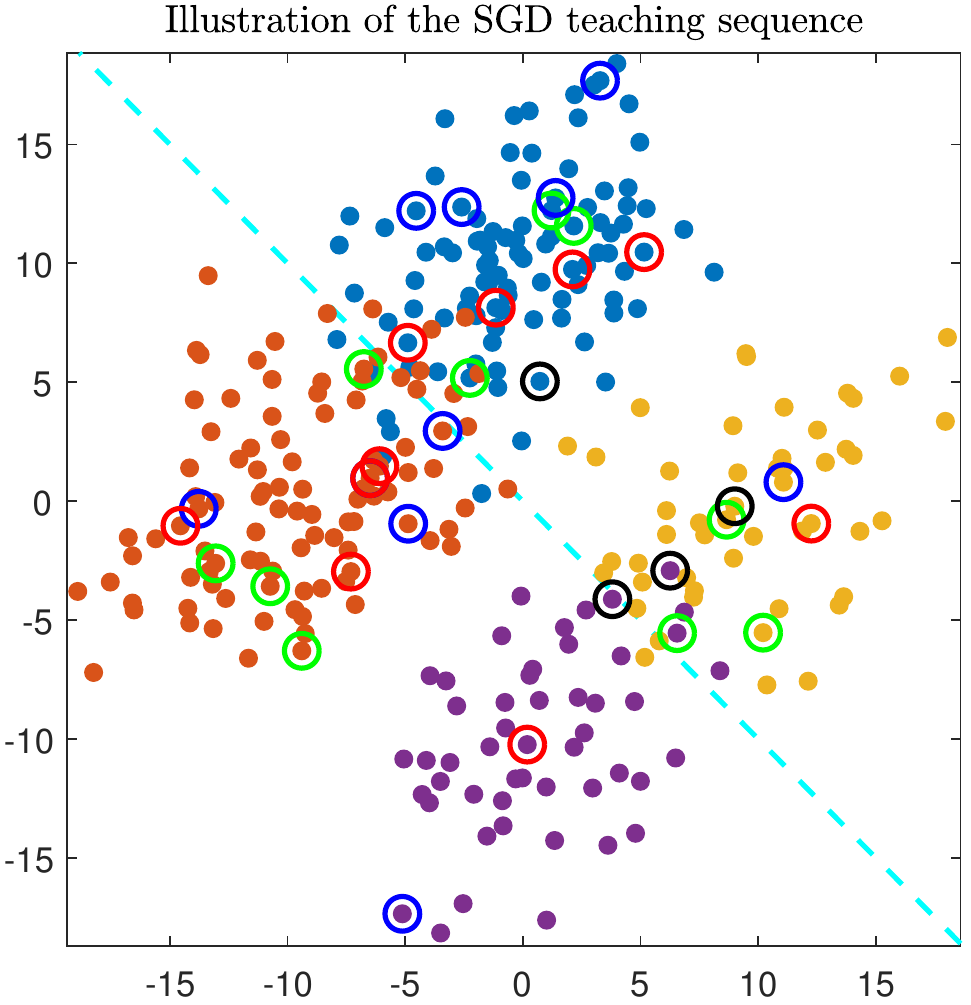}\end{subfigure}&
\begin{subfigure}{0.25\textwidth} \hspace{-9mm} \includegraphics[width=1\columnwidth]{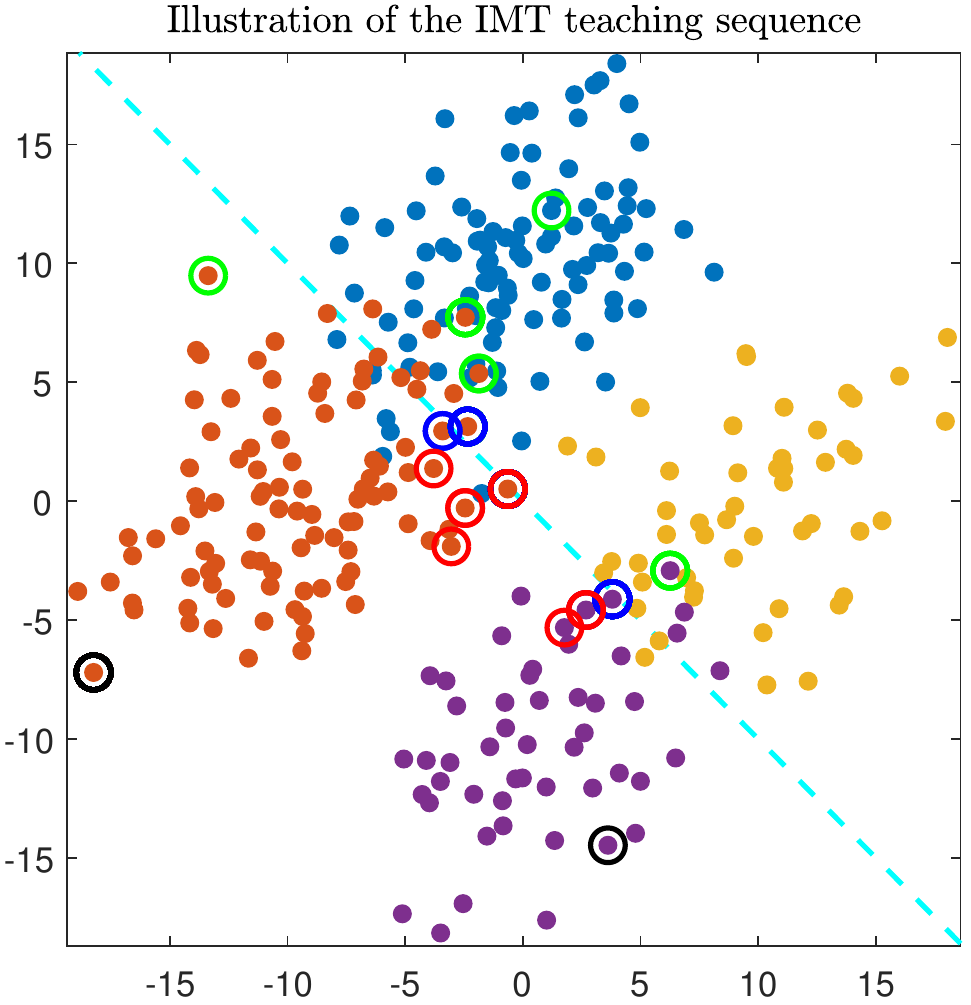}\end{subfigure}&
\begin{subfigure}{0.25\textwidth} \hspace{-12mm} \includegraphics[width=1\columnwidth]{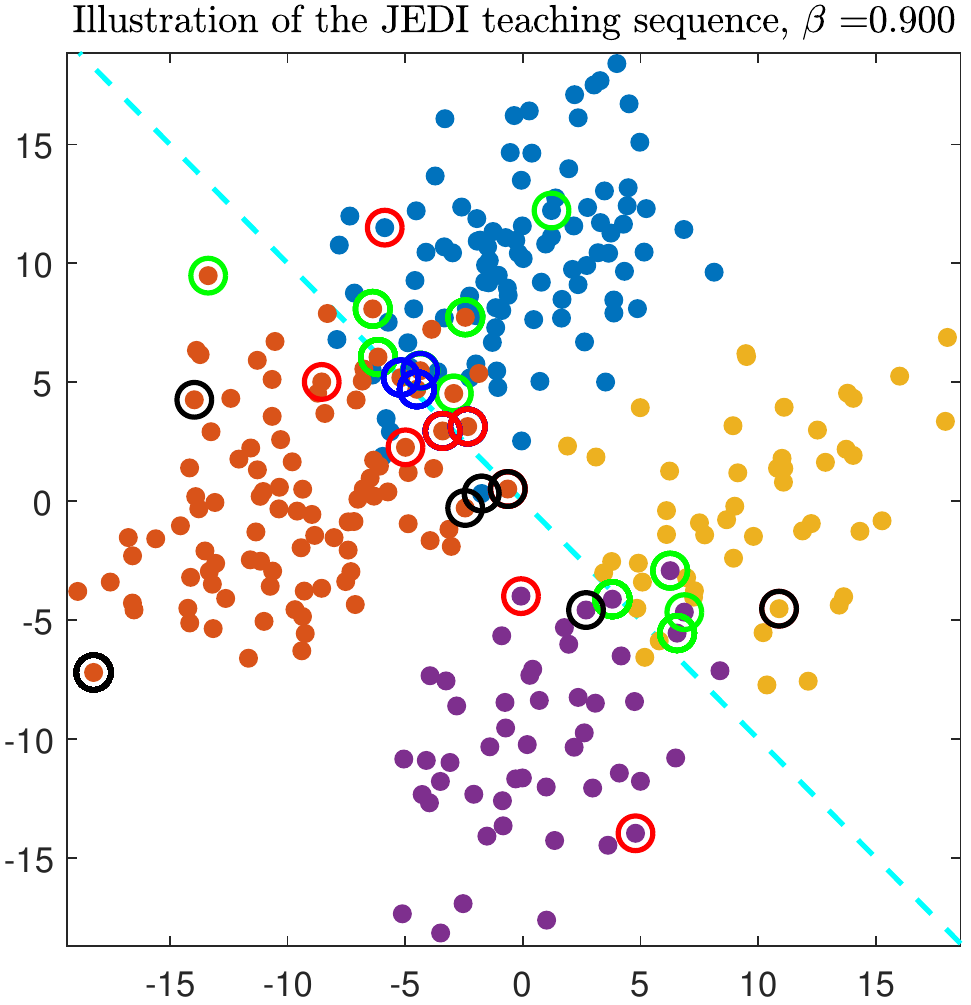}\end{subfigure}\\ 
\end{tabular}
\caption{Left: Comparison of convergence. Right: The unique examples in the teaching sequences of SGD, IMT, and JEDI. Selected examples: Green (iter. 1-100), Blue (iter. 101-200), Red (iter. 201-300), and Black (iter. 301-500). Cyan dashed line: optimal classification hyperplane.}
\label{ToyData}
\end{figure*}
\subsection{Teaching Algorithm}
The details of the JEDI algorithm using harmonic function estimation are provided in \textbf{Algorithm} \ref{secondAlg}. It is given the learner's memory decay rate, target concept, learning rate, teaching set as input, and will output the personalized teaching sequence. It works as follows. We first initialize the iterator $t = 1$ and initial momentum $\textbf{v}_0 = \textbf{0}$. Then, in each teaching iteration, the JEDI teacher estimates the probability of incorrect labeling using harmonic function Eqn. (\ref{harmonic}). Next, the JEDI teacher searches through the teaching set $\Phi$ and finds the example ($\textbf{x}_t,y_t$) that minimizes the objective function in Eqn. (\ref{JEDIharmonic}) which uses the $f_s$ (where $s=1,\hdots,t-1$), and $\frac{1}{f_{-t}}$. Next, the learner performs the labeling on $\textbf{x}_t$ and the JEDI teacher updates affinity matrix $A$, diagonal matrix $D$, and label matrix $F_l$ using the methods described in Section \ref{harmonicSection}. At last, the teacher reveals the true label $y_t$ and the learner performs learning. The JEDI teaching with harmonic function estimation will stop when the maximum number of iterations has been reached. 
\begin{table}[h]
\centering 
\setlength\tabcolsep{1.5pt} 
\scalebox{0.9}{
\hspace{-2mm}
\begin{tabular}{|c|c|c|c|}
\hline 
\textbf{Data set}      & \textbf{\# Examples (Teach)} & \textbf{\# Examples (Evaluate)} & \textbf{\# Features} \\ \hline \hline
10D-Guassian & 400                 & 1600                   & 10          \\ \hline
Comp. vs. Sci & 375                 & 1500                   & 150         \\ \hline
Rec. vs. Talk & 369                 & 1475                   & 150         \\ \hline
\end{tabular}
}
\vspace{2mm}
\caption{Statistics of the three data sets with synthetic learners.}
\label{dataSTAT}
\vspace{-8mm}
\end{table}

\section{Experiments}
In this section, we first conduct the experiments on a toy data set to illustrate the trade-off between diversity and usefulness using JEDI with omniscient teacher. Then, we evaluate the convergence and the performance of JEDI with harmonic function estimation on three data sets using synthetically generated learners. At last, we evaluate the effectiveness of JEDI teaching on two real-world data sets by hiring and teaching a group of crowdsourcing workers.
\subsection{Toy Data Set Visualization}
In order to visualize the selected examples of the teaching sequence, we apply three different teaching methods: SGD, Iterative Machine Teaching (IMT) \cite{DBLP:conf/icml/LiuDHTYSRS17}, and JEDI (omniscient teacher) on a 2D Gaussian mixture data set. This data set is draw from two Gaussian mixture distributions (one positive class and one negative class):
\begin{equation}
    \begin{split}
        p_+(\textbf{x}) &= \frac{2}{3} \mathcal{N}(\textbf{x}|\mu_1, \Sigma_1) + \frac{1}{3} \mathcal{N}(\textbf{x}|\mu_2, \Sigma_2) \\
        p_-(\textbf{x}) &= \frac{2}{3} \mathcal{N}(\textbf{x}|\mu_3, \Sigma_1) + \frac{1}{3} \mathcal{N}(\textbf{x}|\mu_4, \Sigma_2)
    \end{split}
\end{equation}

\noindent where the parameters are $\mu_1 = (0, 8)$, $\mu_2 = (8, 0)$, $\mu_3 = (-8, 0)$, $\mu_4 = (0, -8)$, $\Sigma_1 = [12 \; 6; 6 \; 12]$ and $\Sigma_2 = [10 \; 5; 5 \; 10]$. The number of examples in each class is 150. To guarantee a fair comparison, the target concepts $\textbf{w}_*$ are the same for IMT and JEDI (SGD has no target concept), and the initial concept $\textbf{w}_0$, the first teaching example ($\textbf{x}_1,y_1$), the step size $\eta = 0.03$ are set to be identical for all three methods.

The numbers of unique teaching examples of SGD, IMT, and JEDI are 182, 16, and 27 respectively. As we expected, SGD almost selects half of all examples for teaching. Thus, we visualize the selected examples of SGD with a stride size of 5. From the visualization results, there are several interesting observations. First, the convergence rate of IMT and JEDI are comparable and both have better convergence than SGD. Second, the selected teaching examples of JEDI are much diverse than those of IMT, yet the convergence speed of JEDI is guaranteed. As we can see, in the first 100 iterations (when teaching converges very fast), the unique teaching examples of IMT mainly focus on the upper left two data distributions, but the teaching examples of JEDI are more diversely distributed in all data distributions. The advantages of JEDI is significant especially when the examples are drawn from a mixture distribution because JEDI objective considers both the usefulness and the diversity of the teaching examples. Third, the JEDI selected examples are symmetrically scattered over the optimal classification hyperplane and have more appearances on the data distribution boundaries which reflects our theoretical analysis regarding the exploration and exploitation actions.

\begin{figure*}[!t]
\begin{tabular}{ccc}
\begin{subfigure}{0.33\textwidth} \hspace{-2mm} \includegraphics[width=1\columnwidth]{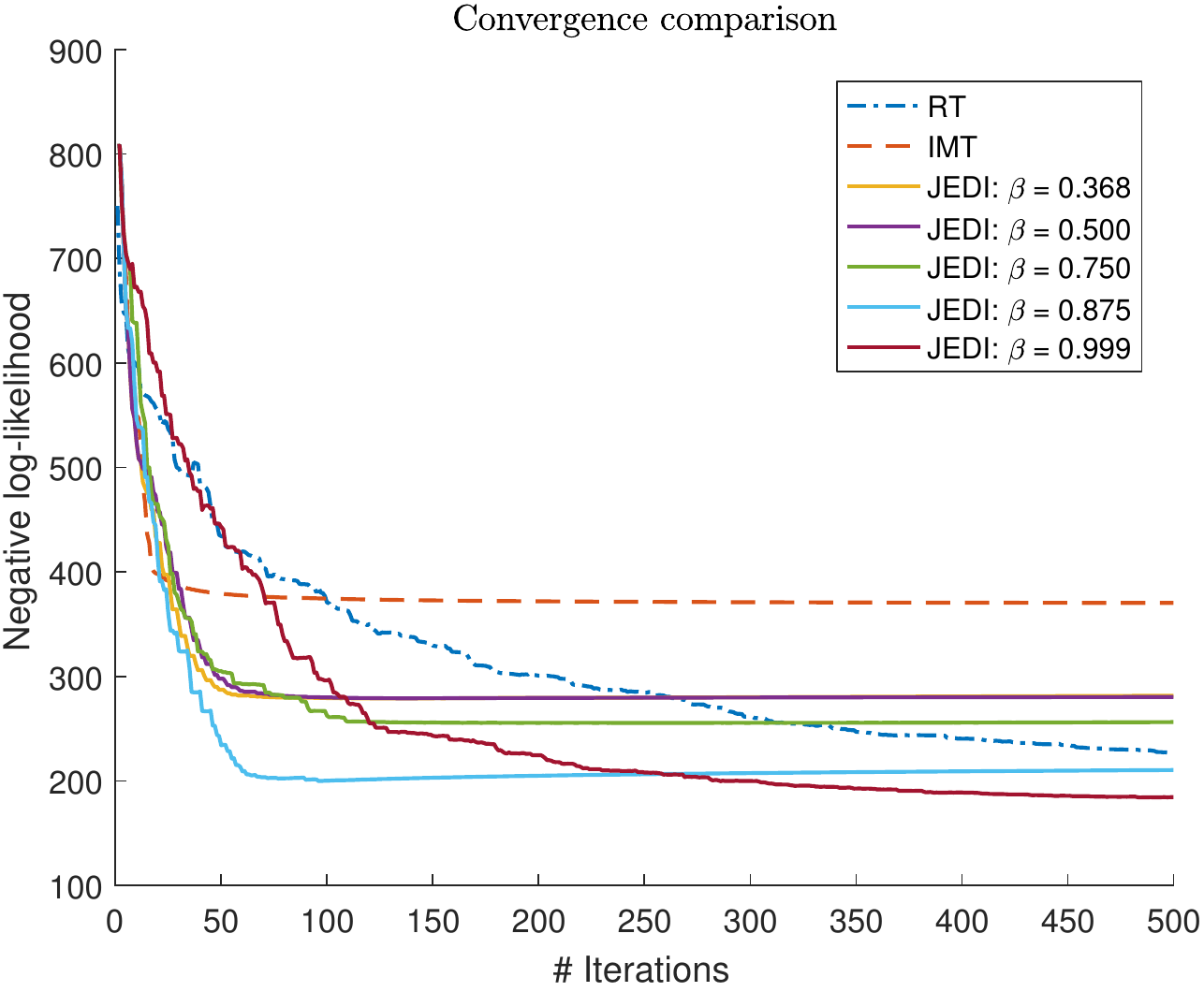}\end{subfigure}&
\begin{subfigure}{0.33\textwidth} \hspace{-4mm} \includegraphics[width=1\columnwidth]{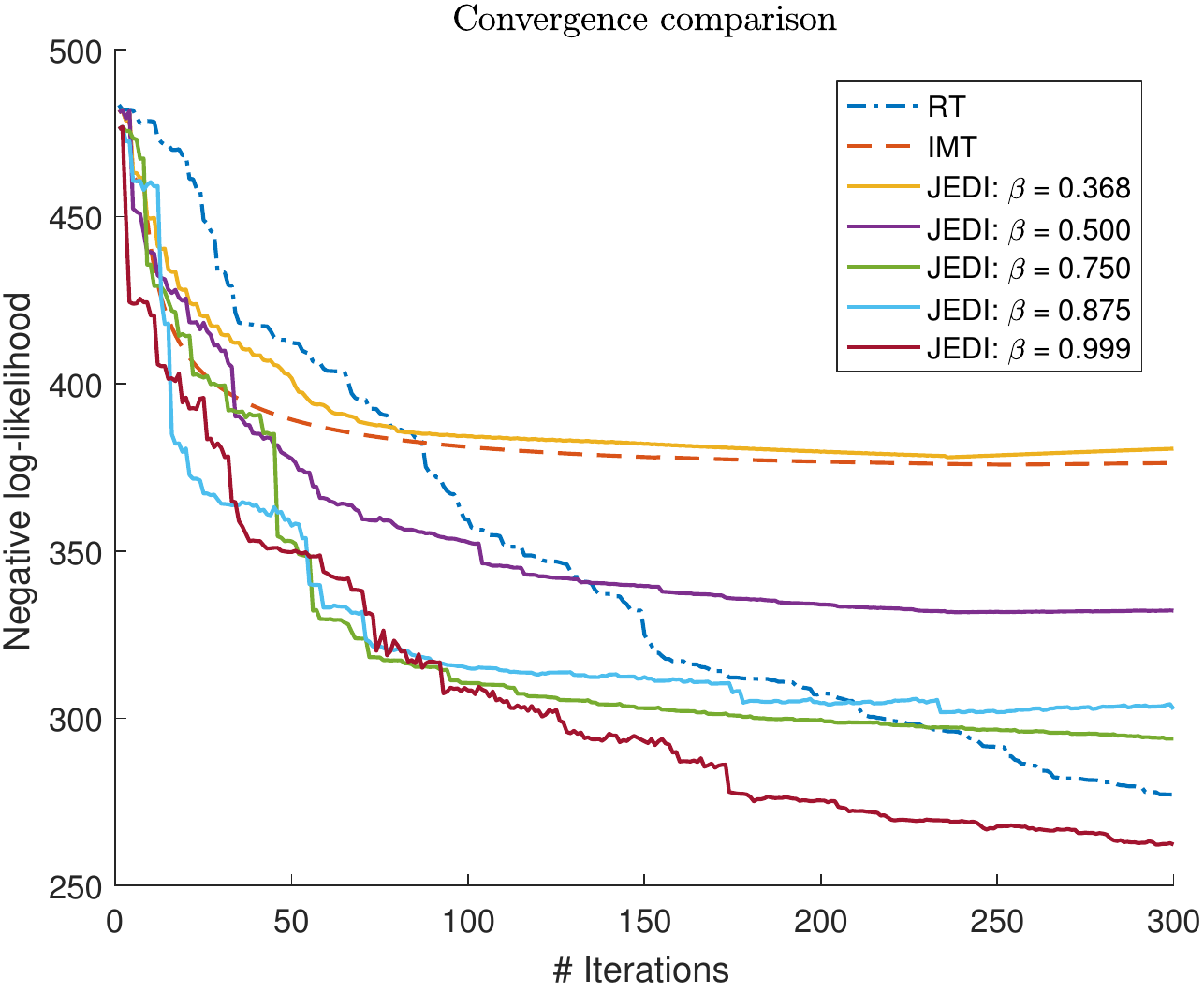}\end{subfigure}&
\begin{subfigure}{0.33\textwidth} \hspace{-6mm} \includegraphics[width=1\columnwidth]{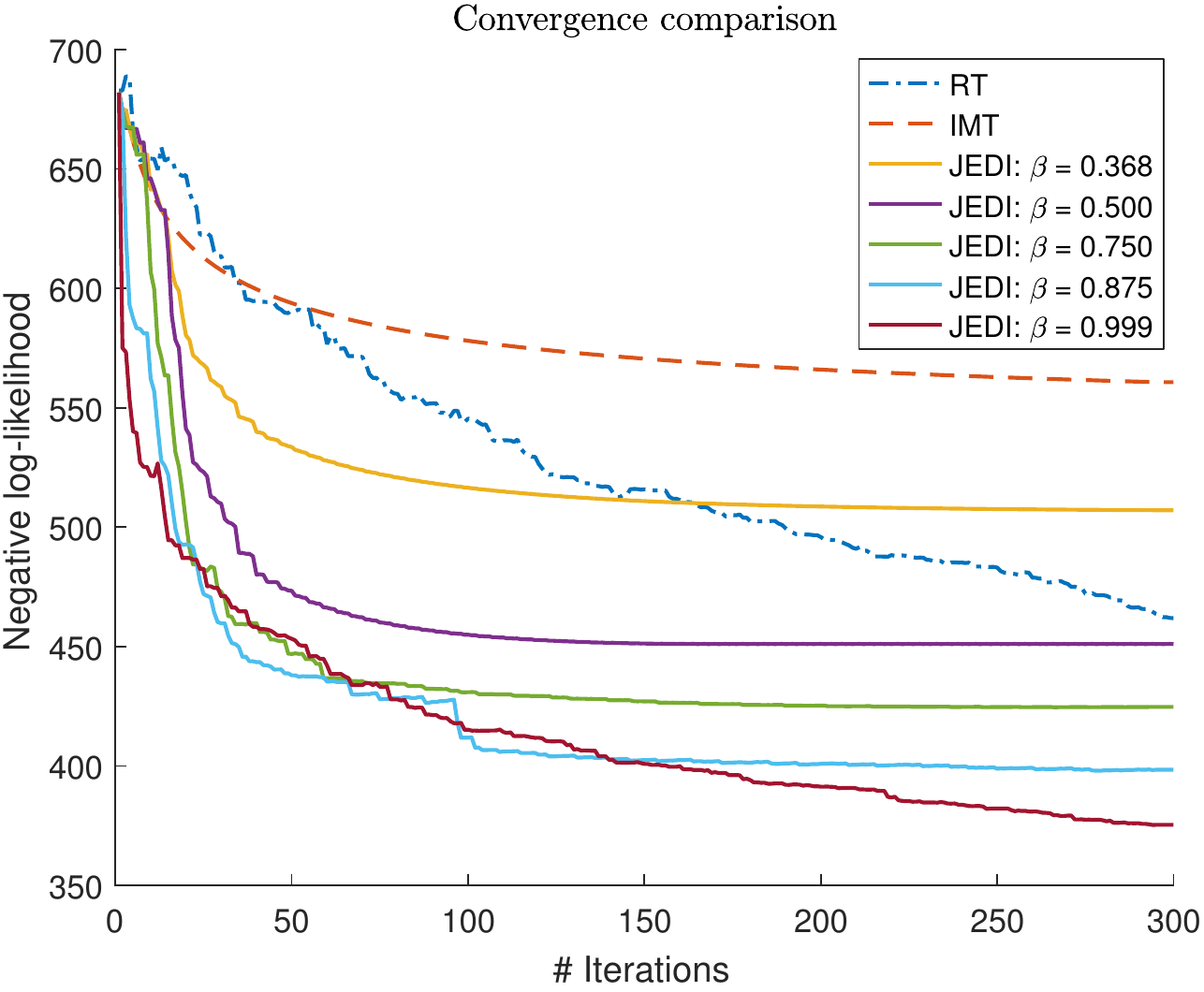}\end{subfigure}\\ 
\end{tabular}
\caption{Convergence plots of data sets: 10D-Gaussian (left), Comp.vs.Sci (middle), Rec.vs.Talk (right)}
\label{SyntheticData}
\end{figure*}

\begin{table*}[h]
\centering
\scalebox{0.9}{
\begin{tabular}{|c|c|c|c|c|c|c|}
\hline
\multirow{2}{*}{} & \multicolumn{2}{c|}{\textbf{10D-Guassian}} & \multicolumn{2}{c|}{\textbf{Comp. vs. Sci}} & \multicolumn{2}{c|}{\textbf{Rec. vs. Talk}} \\ \cline{2-7} 
 & \textit{\begin{tabular}[c]{@{}c@{}}\# Unique Teaching\\ Examples\end{tabular}} & \textit{\begin{tabular}[c]{@{}c@{}}Evaluation \\ Accuracy\end{tabular}} & \textit{\begin{tabular}[c]{@{}c@{}}\# Unique Teaching\\ Examples\end{tabular}} & \textit{\begin{tabular}[c]{@{}c@{}}Evaluation \\ Accuracy\end{tabular}} & \textit{\begin{tabular}[c]{@{}c@{}}\# Unique Teaching\\ Examples\end{tabular}} & \textit{\begin{tabular}[c]{@{}c@{}}Evaluation \\ Accuracy\end{tabular}} \\ \hline
\textbf{RT} & 283 & 0.7625 & 210 & 0.7247 & 208 & 0.6041 \\ \hline
\textbf{IMT} & 13 & 0.7075 & 9 & 0.6887 & 5 & 0.5607 \\ \hline
\textbf{JEDI: $\beta$ = 0.368} & 16 & 0.7575 & 14 & 0.6827 & 12 & 0.5736 \\ \hline
\textbf{JEDI: $\beta$ = 0.500} & 18 & 0.7575 & 16 & 0.6927 & 17 & 0.6102 \\ \hline
\textbf{JEDI: $\beta$ = 0.750} & 30 & 0.7613 & 27 & 0.7207 & 23 & 0.6292 \\ \hline
\textbf{JEDI: $\beta$ = 0.875} & 53 & 0.7775 & 45 & 0.7007 & 34 & 0.6617 \\ \hline
\textbf{JEDI: $\beta$ = 0.999} & 197 & 0.7825 & 64 & 0.7307 & 50 & 0.6915 \\ \hline
\end{tabular}
}
\vspace{2mm}
\caption{Results of three data sets with synthetic learners}
\label{experiment2}
\end{table*}
\subsection{Adaptive Teaching with Synthetic Learners}
The teacher in this group of experiments doesn't have access to the learners' concepts. Thus, teachers of Random Teaching (RT), IMT, and JEDI can only estimate a learner's concept using harmonic function. We evaluates these methods on three data sets, as shown in Table \ref{dataSTAT}, which include a 10-dimensional Gaussian data set and two text data sets from 20 Newsgroups. The learners are randomly generated as a vector that has the same length as the example features. After performing the learning, a random Gaussian noise will be added to the learner's concept vector to simulate the learning uncertainty. Below are the settings of the data sets:
\begin{itemize}
    \item \textbf{10D-Gaussian}: The means are $\mu_1 = (-0.6, \hdots, -0.6)$, $\mu_2 = (0.6, \hdots, 0.6)$ and the diagonal of its covariance matrix has random values between 1 to 10. Initial step size is set to $\eta_0 = 0.03$ and it is gradually decreased as $\eta_t = \frac{20}{20+t} \eta_0$ where $t$ is the teaching iterations. 
    \item \textbf{Comp.vs.Sci} and \textbf{Rec.vs.Talk}: We use their largest classification tasks for each of them. The extracted features are TF-IDF. Initial step size is set to $\eta_0 = 0.03$ and it is gradually decreased as $\eta_t = \frac{200}{200+t} \eta_0$. 
\end{itemize}

All three data sets are randomly split into $20\%$ as the teaching set (has true labels for the teacher) and $80\%$ as the evaluation set. For the JEDI teacher, we have five different synthetic learners with $\beta \in \{0.368, 0.5, 0.75, 0.875, 0.999 \}$ and these values represent for learner with memory window size of $\{1, 2, 4, 8, \mathrm{Inf} \}$. To guarantee a fair comparison, all learners have the same initial concept $\textbf{w}_0$, first teaching example $(\textbf{x}_1, y_1)$, and target concept $\textbf{w}_*$ (except RT, which doesn't have target concept). As we can see from the experiment results in Figure \ref{SyntheticData} and Table \ref{experiment2}, the IMT and JEDI have consistently better convergence speed than RT. We also observe that the JEDI learners with exponential decay memories usually outperforms the no memory learners of IMT in terms of either the convergence speed or the evaluation accuracy. Meanwhile, as we expected, the JEDI learners with larger $\beta$ has slower convergence speed, increasing number of unique teaching examples, and possibly better performance in the evaluation. 

\subsection{Adaptive Teaching with Real Human Learners}
The teaching experiments with real human learners are designed for crowdsourcing workers to learn the concept of labeling different animals images \cite{DBLP:conf/sdm/ZhouYH17} based on the animal breed. We utilize two categories of images (Cat and Canidae) and the label of each image is either \textit{domestic} or \textit{wild}. Following the same convention of \cite{DBLP:conf/sdm/ZhouYH17}, each image is represented by the top 110 TF-IDF features using bag-of-visual-words extracted from a three-level image pyramid. The experiment is designed to have three modules: \textit{memory length estimation, interactive teaching}, and \textit{performance evaluation}. 

\vspace{2mm}
\noindent \textbf{Memory length estimation}: For each crowdsourcing learner, his/her memory decay rate $\beta$ is not available to the JEDI teacher beforehand. Thus, we propose to use the images sorting task to estimate each learner's $\beta$. This sorting task shows increasing number (from 2 to 9) of randomly ordered images to the learner for a few seconds (from 3s to 10s), and then ask the learner to recover the correct order of these images after random shuffling. Each learner's memory decay rate is ad-hocly estimated as $\beta = 1 - \frac{1}{ \overline{n} }$ where $\overline{n}$ is the mean of the maximum number of ordered images that this learner can recover. There are three memory length estimation trials, we drop the one with smallest memory length and take the mean of the remaining two. 

\vspace{2mm}
\noindent \textbf{Interactive teaching}: In total, we hired 58 crowdsourcing workers (30 for cat data set, 28 for canidae data set). All human workers are graduate students who are hired from Arizona State University and have machine learning background. Besides RT, IMT, and JEDI, we also add the Expected Error Reduction (EER \cite{DBLP:conf/cvpr/JohnsAB15}) as a comparison teaching method which is specially designed for teaching real human learners. To guarantee a fair comparison, each learner will be assigned with one of these four teachers using round-robin scheduling. The numbers of teaching images of JEDI are 20, 30, or 40 if $\overline{n}$ falls into these ranges $[2, 4.5], (4.5, 6.5], (6.5, 9]$ respectively. The numbers of teaching images of RT, IMT, and EER are fixed as 30. It should be noticed that the ideal number of teaching examples for different teaching tasks could have a large variation due to the various learning abilities of the learners, different scales of the data set, etc. We have left the exploration of this specific setting to the future work. To deal with the "cold start" issue, the first five teaching examples are randomly selected from the teaching set. All workers know that they will be taught, and the incentive for workers to learn is by giving double payment if they have the top 20 percent labeling accuracies among all workers. 

\vspace{2mm}
\noindent \textbf{Performance evaluation}: For each crowdsourcing worker, they are asked to label 100 images (50 domestic/50 wild) in the evaluation stage. The purpose of teaching is to let the human learner grasp the idea of this domestic/wild classification concept. Thus, we propose to use \textit{teaching gain} as the evaluation metric which is defined as the labeling accuracy during evaluation minus the labeling accuracy (of these first seen teaching examples) during teaching. From the plots shown in Figure \ref{realdata}, we observe that the teachers with an explicit learner's model (e.g. IMT and JEDI) performs better than model-agnostic teachers and the JEDI teacher consistently performs the best over all teaching strategies. Interestingly, we also see that JEDI is the only teacher that has positive teaching gain on the cat data set. One possible explanation is that cat breed classification is a very difficult task and these crowdsourcing workers without given a properly selected teaching sequence could hardly grasp this labeling concept. On the other hand, the human learners have difficulties to visually differentiate the wild and domestic cats is because of that the subtle difference of their discriminative features could be easily forgotten by human learners. However, JEDI is the only model that could capture this memory loss by assuming that each learner's memory has an exponentially decayed rate. 
\begin{figure}[!t]
\begin{tabular}{cc}
\begin{subfigure}{0.23\textwidth}\hspace{-2mm} \includegraphics[width=1\columnwidth]{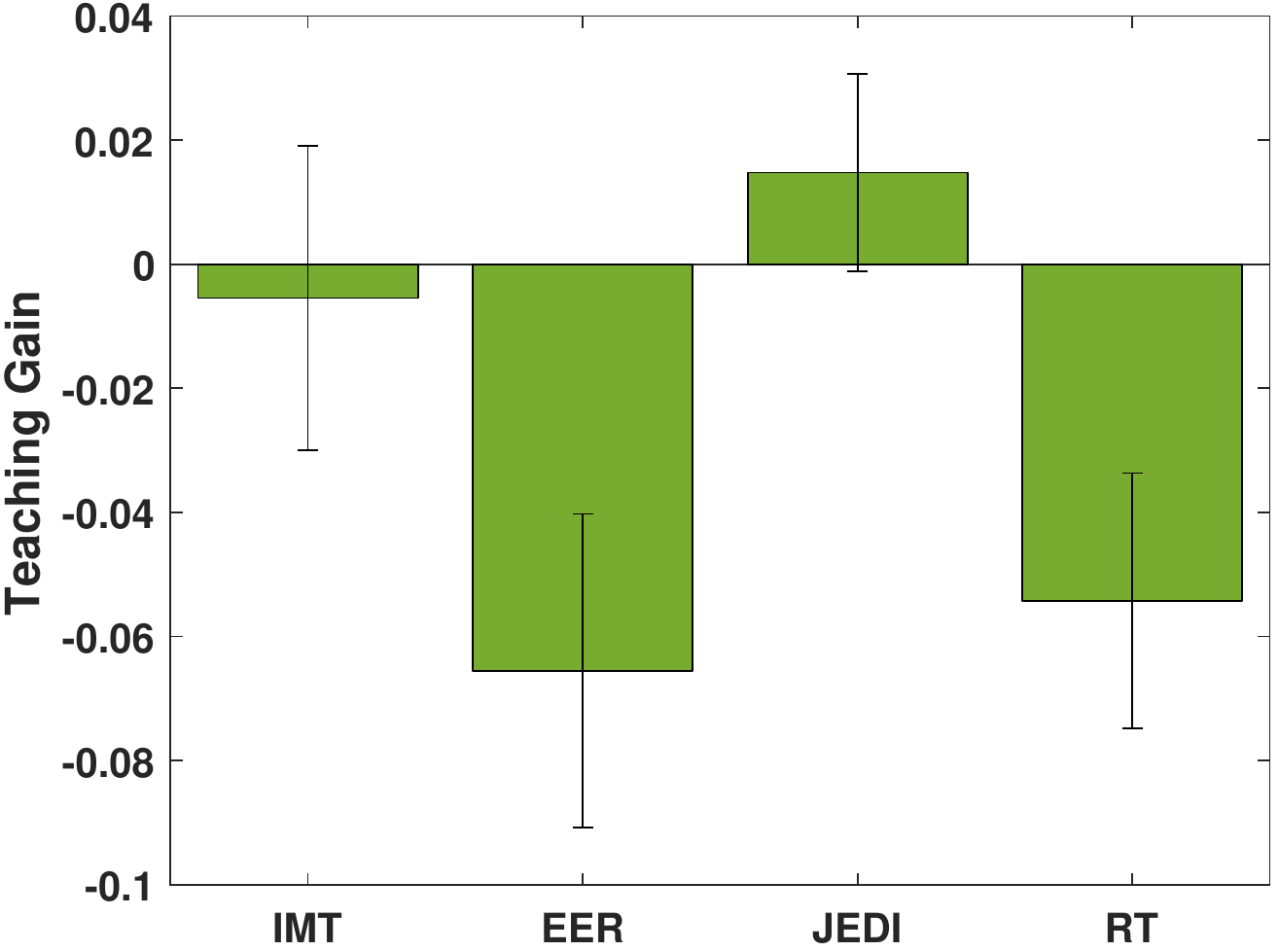}\label{cat}\end{subfigure}&
\begin{subfigure}{0.23\textwidth}\hspace{-4mm} \includegraphics[width=1\columnwidth]{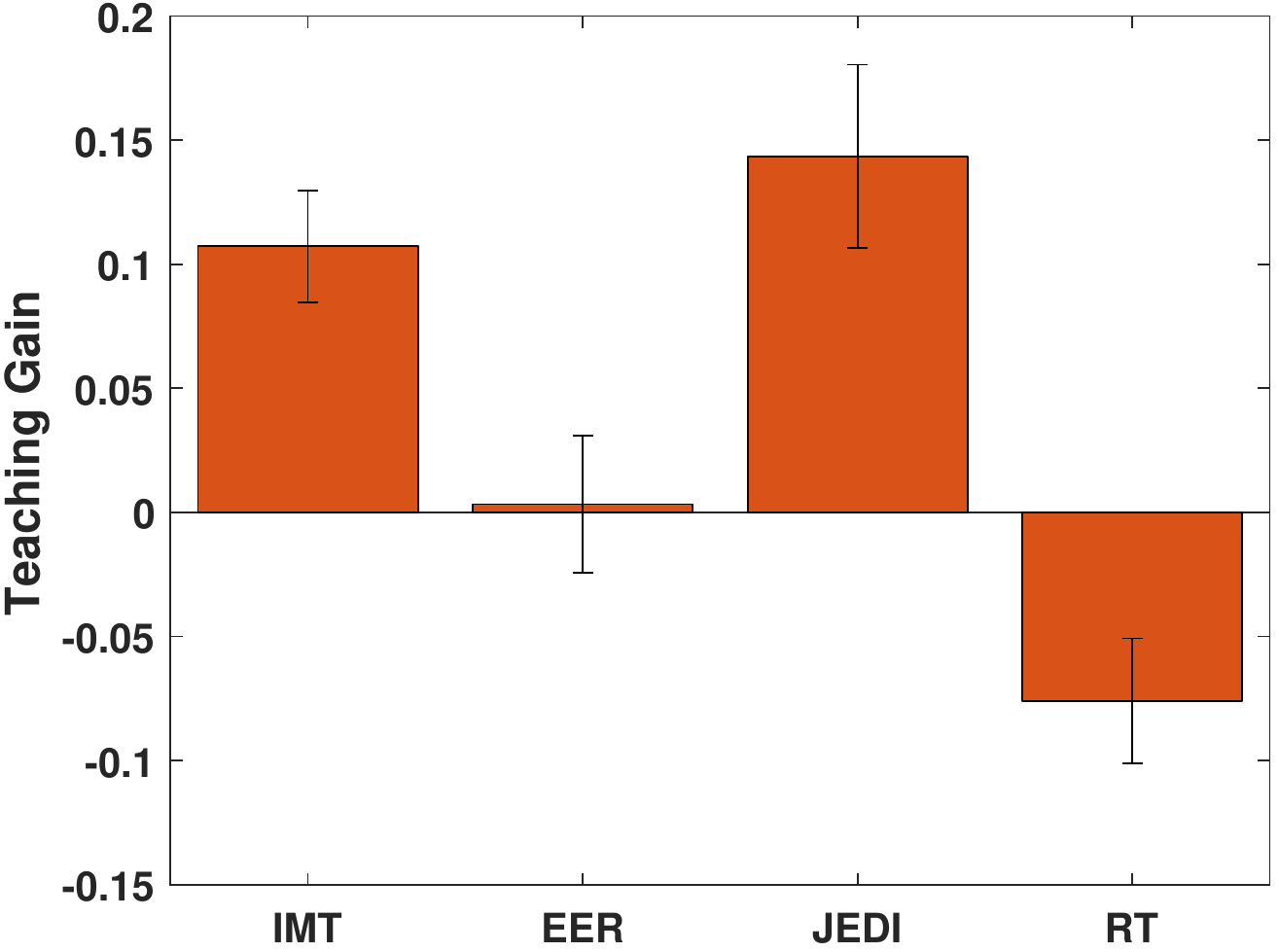}\label{dog}\end{subfigure}\\ 
\end{tabular}
\caption{Teaching Gain of Cat (left) and Canidae (right)}
\label{realdata}
\end{figure}
\section{Conclusion}
In this paper, we study the problem of crowd teaching which applies the machine teaching paradigm into the crowdsourcing applications. The proposed JEDI teaching framework advances the state-of-the-art techniques in multiple dimensions in terms of human memory decay modeling, converging speed, teaching comprehensiveness and teaching accuracy, etc. The experimental results on several data sets with synthetic learners and crowdsourcing workers show the superiority of JEDI teaching. Future work could focus on multiple directions. Adapting to the multi-class teaching or finding an alternative method to perform concept estimation are two possible straightforward extensions. Furthermore, other promising explorations could be teaching the gray-box learners (using different loss functions or different learning procedures), teaching the black-box learners (model agnostic) and teaching with human interpretable explanations (e.g., area of interest on images or key phrases of documents).


\begin{acks}
This work is supported by National Science Foundation under Grant No. IIS-1552654, Grant No. IIS-1813464 and Grant No. CNS-1629888, the U.S. Department of Homeland Security under Grant Award Number 2017-ST-061-QA0001, and an IBM Faculty Award. The views and conclusions are those of the authors and should not be interpreted as representing the official policies of the funding agencies or the government.
\end{acks}

\vspace{0.5mm}
\bibliographystyle{ACM-Reference-Format}
\bibliography{JEDI.bib}


\begin{thebibliography}{29}


\ifx \showCODEN    \undefined \def \showCODEN     #1{\unskip}     \fi
\ifx \showDOI      \undefined \def \showDOI       #1{#1}\fi
\ifx \showISBNx    \undefined \def \showISBNx     #1{\unskip}     \fi
\ifx \showISBNxiii \undefined \def \showISBNxiii  #1{\unskip}     \fi
\ifx \showISSN     \undefined \def \showISSN      #1{\unskip}     \fi
\ifx \showLCCN     \undefined \def \showLCCN      #1{\unskip}     \fi
\ifx \shownote     \undefined \def \shownote      #1{#1}          \fi
\ifx \showarticletitle \undefined \def \showarticletitle #1{#1}   \fi
\ifx \showURL      \undefined \def \showURL       {\relax}        \fi
\providecommand\bibfield[2]{#2}
\providecommand\bibinfo[2]{#2}
\providecommand\natexlab[1]{#1}
\providecommand\showeprint[2][]{arXiv:#2}

\bibitem[\protect\citeauthoryear{Bengio, Louradour, Collobert, and
  Weston}{Bengio et~al\mbox{.}}{2009}]%
        {DBLP:conf/icml/BengioLCW09}
\bibfield{author}{\bibinfo{person}{Yoshua Bengio},
  \bibinfo{person}{J{\'{e}}r{\^{o}}me Louradour}, \bibinfo{person}{Ronan
  Collobert}, {and} \bibinfo{person}{Jason Weston}.}
  \bibinfo{year}{2009}\natexlab{}.
\newblock \showarticletitle{Curriculum learning}. In
  \bibinfo{booktitle}{\emph{ICML}}. \bibinfo{pages}{41--48}.
\newblock


\bibitem[\protect\citeauthoryear{Doliwa, Fan, Simon, and Zilles}{Doliwa
  et~al\mbox{.}}{2014}]%
        {DBLP:journals/jmlr/DoliwaFSZ14}
\bibfield{author}{\bibinfo{person}{Thorsten Doliwa}, \bibinfo{person}{Gaojian
  Fan}, \bibinfo{person}{Hans~Ulrich Simon}, {and} \bibinfo{person}{Sandra
  Zilles}.} \bibinfo{year}{2014}\natexlab{}.
\newblock \showarticletitle{Recursive teaching dimension, VC-dimension and
  sample compression}.
\newblock \bibinfo{journal}{\emph{Journal of Machine Learning Research}}
  \bibinfo{volume}{15}, \bibinfo{number}{1} (\bibinfo{year}{2014}),
  \bibinfo{pages}{3107--3131}.
\newblock


\bibitem[\protect\citeauthoryear{Felder and Silverman}{Felder and
  Silverman}{1988}]%
        {Felder88learningand}
\bibfield{author}{\bibinfo{person}{Richard~M. Felder} {and}
  \bibinfo{person}{Linda~K. Silverman}.} \bibinfo{year}{1988}\natexlab{}.
\newblock \showarticletitle{Learning and teaching styles in engineering
  education}.
\newblock \bibinfo{journal}{\emph{Engineering Education}}
  (\bibinfo{year}{1988}).
\newblock


\bibitem[\protect\citeauthoryear{Gigu{\`e}re and Love}{Gigu{\`e}re and
  Love}{2013}]%
        {memoryretrieval}
\bibfield{author}{\bibinfo{person}{Gyslain Gigu{\`e}re} {and}
  \bibinfo{person}{Bradley~C. Love}.} \bibinfo{year}{2013}\natexlab{}.
\newblock \showarticletitle{Limits in decision making arise from limits in
  memory retrieval}.
\newblock \bibinfo{journal}{\emph{Proceedings of the National Academy of
  Sciences}} \bibinfo{volume}{110}, \bibinfo{number}{19}
  (\bibinfo{year}{2013}), \bibinfo{pages}{7613--7618}.
\newblock


\bibitem[\protect\citeauthoryear{Goldman and Kearns}{Goldman and
  Kearns}{1995}]%
        {DBLP:journals/jcss/GoldmanK95}
\bibfield{author}{\bibinfo{person}{Sally~A. Goldman} {and}
  \bibinfo{person}{Michael~J. Kearns}.} \bibinfo{year}{1995}\natexlab{}.
\newblock \showarticletitle{On the Complexity of Teaching}.
\newblock \bibinfo{journal}{\emph{J. Comput. System Sci.}}
  \bibinfo{volume}{50}, \bibinfo{number}{1} (\bibinfo{year}{1995}),
  \bibinfo{pages}{20--31}.
\newblock


\bibitem[\protect\citeauthoryear{Henson}{Henson}{1998}]%
        {HENSON199873}
\bibfield{author}{\bibinfo{person}{Richard~N.A. Henson}.}
  \bibinfo{year}{1998}\natexlab{}.
\newblock \showarticletitle{Short-Term Memory for Serial Order: The Start-End
  Model}.
\newblock \bibinfo{journal}{\emph{Cognitive Psychology}} \bibinfo{volume}{36},
  \bibinfo{number}{2} (\bibinfo{year}{1998}), \bibinfo{pages}{73 -- 137}.
\newblock


\bibitem[\protect\citeauthoryear{Jiang, Meng, Yu, Lan, Shan, and
  Hauptmann}{Jiang et~al\mbox{.}}{2014}]%
        {DBLP:conf/nips/JiangMYLSH14}
\bibfield{author}{\bibinfo{person}{Lu Jiang}, \bibinfo{person}{Deyu Meng},
  \bibinfo{person}{Shoou{-}I Yu}, \bibinfo{person}{Zhen{-}Zhong Lan},
  \bibinfo{person}{Shiguang Shan}, {and} \bibinfo{person}{Alexander~G.
  Hauptmann}.} \bibinfo{year}{2014}\natexlab{}.
\newblock \showarticletitle{Self-Paced Learning with Diversity}. In
  \bibinfo{booktitle}{\emph{NIPS}}. \bibinfo{pages}{2078--2086}.
\newblock


\bibitem[\protect\citeauthoryear{Johns, {Mac Aodha}, and Brostow}{Johns
  et~al\mbox{.}}{2015}]%
        {DBLP:conf/cvpr/JohnsAB15}
\bibfield{author}{\bibinfo{person}{Edward Johns}, \bibinfo{person}{Oisin {Mac
  Aodha}}, {and} \bibinfo{person}{Gabriel~J. Brostow}.}
  \bibinfo{year}{2015}\natexlab{}.
\newblock \showarticletitle{Becoming the expert - interactive multi-class
  machine teaching}. In \bibinfo{booktitle}{\emph{CVPR}}.
  \bibinfo{pages}{2616--2624}.
\newblock


\bibitem[\protect\citeauthoryear{Khan, Zhu, and Mutlu}{Khan
  et~al\mbox{.}}{2011}]%
        {DBLP:conf/nips/KhanZM11}
\bibfield{author}{\bibinfo{person}{Faisal Khan}, \bibinfo{person}{Xiaojin Zhu},
  {and} \bibinfo{person}{Bilge Mutlu}.} \bibinfo{year}{2011}\natexlab{}.
\newblock \showarticletitle{How Do Humans Teach: On Curriculum Learning and
  Teaching Dimension}. In \bibinfo{booktitle}{\emph{NIPS}}.
  \bibinfo{pages}{1449--1457}.
\newblock


\bibitem[\protect\citeauthoryear{Kingma and Ba}{Kingma and Ba}{2014}]%
        {DBLP:journals/corr/KingmaB14}
\bibfield{author}{\bibinfo{person}{Diederik~P. Kingma} {and}
  \bibinfo{person}{Jimmy Ba}.} \bibinfo{year}{2014}\natexlab{}.
\newblock \showarticletitle{Adam: {A} Method for Stochastic Optimization}.
\newblock \bibinfo{journal}{\emph{CoRR}}  \bibinfo{volume}{abs/1412.6980}
  (\bibinfo{year}{2014}).
\newblock


\bibitem[\protect\citeauthoryear{Liu and Zhu}{Liu and Zhu}{2016}]%
        {DBLP:journals/jmlr/LiuZ16}
\bibfield{author}{\bibinfo{person}{Ji Liu} {and} \bibinfo{person}{Xiaojin
  Zhu}.} \bibinfo{year}{2016}\natexlab{}.
\newblock \showarticletitle{The Teaching Dimension of Linear Learners}.
\newblock \bibinfo{journal}{\emph{Journal of Machine Learning Research}}
  \bibinfo{volume}{17} (\bibinfo{year}{2016}), \bibinfo{pages}{162:1--162:25}.
\newblock


\bibitem[\protect\citeauthoryear{Liu, Peng, and Ihler}{Liu
  et~al\mbox{.}}{2012}]%
        {DBLP:conf/nips/LiuPI12}
\bibfield{author}{\bibinfo{person}{Qiang Liu}, \bibinfo{person}{Jian Peng},
  {and} \bibinfo{person}{Alexander~T. Ihler}.} \bibinfo{year}{2012}\natexlab{}.
\newblock \showarticletitle{Variational Inference for Crowdsourcing}. In
  \bibinfo{booktitle}{\emph{NIPS}}. \bibinfo{pages}{701--709}.
\newblock


\bibitem[\protect\citeauthoryear{Liu, Dai, Humayun, Tay, Yu, Smith, Rehg, and
  Song}{Liu et~al\mbox{.}}{2017}]%
        {DBLP:conf/icml/LiuDHTYSRS17}
\bibfield{author}{\bibinfo{person}{Weiyang Liu}, \bibinfo{person}{Bo Dai},
  \bibinfo{person}{Ahmad Humayun}, \bibinfo{person}{Charlene Tay},
  \bibinfo{person}{Chen Yu}, \bibinfo{person}{Linda~B. Smith},
  \bibinfo{person}{James~M. Rehg}, {and} \bibinfo{person}{Le Song}.}
  \bibinfo{year}{2017}\natexlab{}.
\newblock \showarticletitle{Iterative Machine Teaching}. In
  \bibinfo{booktitle}{\emph{ICML}}. \bibinfo{pages}{2149--2158}.
\newblock


\bibitem[\protect\citeauthoryear{Loftus}{Loftus}{1985}]%
        {forgettingcurve}
\bibfield{author}{\bibinfo{person}{Geoffrey~R. Loftus}.}
  \bibinfo{year}{1985}\natexlab{}.
\newblock \showarticletitle{Evaluating Forgetting Curves}.
\newblock \bibinfo{journal}{\emph{Journal of Experimental Psychology: Learning,
  Memory, and Cognition}} \bibinfo{volume}{11}, \bibinfo{number}{2}
  (\bibinfo{year}{1985}), \bibinfo{pages}{397--406}.
\newblock


\bibitem[\protect\citeauthoryear{Mei and Zhu}{Mei and Zhu}{2015}]%
        {DBLP:conf/aaai/MeiZ15}
\bibfield{author}{\bibinfo{person}{Shike Mei} {and} \bibinfo{person}{Xiaojin
  Zhu}.} \bibinfo{year}{2015}\natexlab{}.
\newblock \showarticletitle{Using Machine Teaching to Identify Optimal
  Training-Set Attacks on Machine Learners}. In
  \bibinfo{booktitle}{\emph{AAAI}}. \bibinfo{pages}{2871--2877}.
\newblock


\bibitem[\protect\citeauthoryear{Patil, Zhu, Kopec, and Love}{Patil
  et~al\mbox{.}}{2014}]%
        {DBLP:conf/nips/PatilZKL14}
\bibfield{author}{\bibinfo{person}{Kaustubh~R. Patil}, \bibinfo{person}{Xiaojin
  Zhu}, \bibinfo{person}{Lukasz Kopec}, {and} \bibinfo{person}{Bradley~C.
  Love}.} \bibinfo{year}{2014}\natexlab{}.
\newblock \showarticletitle{Optimal Teaching for Limited-Capacity Human
  Learners}. In \bibinfo{booktitle}{\emph{NIPS}}. \bibinfo{pages}{2465--2473}.
\newblock


\bibitem[\protect\citeauthoryear{Shah and Zhou}{Shah and Zhou}{2016}]%
        {DBLP:conf/icml/ShahZ16}
\bibfield{author}{\bibinfo{person}{Nihar~B. Shah} {and}
  \bibinfo{person}{Dengyong Zhou}.} \bibinfo{year}{2016}\natexlab{}.
\newblock \showarticletitle{No Oops, You Won't Do It Again: Mechanisms for
  Self-correction in Crowdsourcing}. In \bibinfo{booktitle}{\emph{ICML}}.
  \bibinfo{pages}{1--10}.
\newblock


\bibitem[\protect\citeauthoryear{Shah, Zhou, and Peres}{Shah
  et~al\mbox{.}}{2015}]%
        {DBLP:conf/icml/ShahZP15}
\bibfield{author}{\bibinfo{person}{Nihar~B. Shah}, \bibinfo{person}{Dengyong
  Zhou}, {and} \bibinfo{person}{Yuval Peres}.} \bibinfo{year}{2015}\natexlab{}.
\newblock \showarticletitle{Approval Voting and Incentives in Crowdsourcing}.
  In \bibinfo{booktitle}{\emph{ICML}}. \bibinfo{pages}{10--19}.
\newblock


\bibitem[\protect\citeauthoryear{Singla, Bogunovic, Bart{\'{o}}k, Karbasi, and
  Krause}{Singla et~al\mbox{.}}{2014}]%
        {DBLP:conf/icml/SinglaBBKK14}
\bibfield{author}{\bibinfo{person}{Adish Singla}, \bibinfo{person}{Ilija
  Bogunovic}, \bibinfo{person}{G{\'{a}}bor Bart{\'{o}}k}, \bibinfo{person}{Amin
  Karbasi}, {and} \bibinfo{person}{Andreas Krause}.}
  \bibinfo{year}{2014}\natexlab{}.
\newblock \showarticletitle{Near-Optimally Teaching the Crowd to Classify}. In
  \bibinfo{booktitle}{\emph{ICML}}. \bibinfo{pages}{154--162}.
\newblock


\bibitem[\protect\citeauthoryear{Xiao, Biggio, Brown, Fumera, Eckert, and
  Roli}{Xiao et~al\mbox{.}}{2015}]%
        {DBLP:conf/icml/XiaoBBFER15}
\bibfield{author}{\bibinfo{person}{Huang Xiao}, \bibinfo{person}{Battista
  Biggio}, \bibinfo{person}{Gavin Brown}, \bibinfo{person}{Giorgio Fumera},
  \bibinfo{person}{Claudia Eckert}, {and} \bibinfo{person}{Fabio Roli}.}
  \bibinfo{year}{2015}\natexlab{}.
\newblock \showarticletitle{Is Feature Selection Secure against Training Data
  Poisoning?}. In \bibinfo{booktitle}{\emph{ICML}}.
  \bibinfo{pages}{1689--1698}.
\newblock


\bibitem[\protect\citeauthoryear{Yan, Rosales, Fung, and Dy}{Yan
  et~al\mbox{.}}{2011}]%
        {DBLP:conf/icml/YanRFD11}
\bibfield{author}{\bibinfo{person}{Yan Yan}, \bibinfo{person}{R{\'{o}}mer
  Rosales}, \bibinfo{person}{Glenn Fung}, {and} \bibinfo{person}{Jennifer~G.
  Dy}.} \bibinfo{year}{2011}\natexlab{}.
\newblock \showarticletitle{Active Learning from Crowds}. In
  \bibinfo{booktitle}{\emph{ICML}}. \bibinfo{pages}{1161--1168}.
\newblock


\bibitem[\protect\citeauthoryear{Zhou, Bousquet, Lal, Weston, and
  Sch{\"{o}}lkopf}{Zhou et~al\mbox{.}}{2003}]%
        {DBLP:conf/nips/ZhouBLWS03}
\bibfield{author}{\bibinfo{person}{Dengyong Zhou}, \bibinfo{person}{Olivier
  Bousquet}, \bibinfo{person}{Thomas~Navin Lal}, \bibinfo{person}{Jason
  Weston}, {and} \bibinfo{person}{Bernhard Sch{\"{o}}lkopf}.}
  \bibinfo{year}{2003}\natexlab{}.
\newblock \showarticletitle{Learning with Local and Global Consistency}. In
  \bibinfo{booktitle}{\emph{NIPS}}. \bibinfo{pages}{321--328}.
\newblock


\bibitem[\protect\citeauthoryear{Zhou, Liu, Platt, Meek, and Shah}{Zhou
  et~al\mbox{.}}{2015}]%
        {DBLP:journals/corr/ZhouLPMS15}
\bibfield{author}{\bibinfo{person}{Dengyong Zhou}, \bibinfo{person}{Qiang Liu},
  \bibinfo{person}{John~C. Platt}, \bibinfo{person}{Christopher Meek}, {and}
  \bibinfo{person}{Nihar~B. Shah}.} \bibinfo{year}{2015}\natexlab{}.
\newblock \showarticletitle{Regularized Minimax Conditional Entropy for
  Crowdsourcing}.
\newblock \bibinfo{journal}{\emph{CoRR}}  \bibinfo{volume}{abs/1503.07240}
  (\bibinfo{year}{2015}).
\newblock


\bibitem[\protect\citeauthoryear{Zhou and He}{Zhou and He}{2016}]%
        {TAC}
\bibfield{author}{\bibinfo{person}{Yao Zhou} {and} \bibinfo{person}{Jingrui
  He}.} \bibinfo{year}{2016}\natexlab{}.
\newblock \showarticletitle{Crowdsourcing via Tensor Augmentation and
  Completion}. In \bibinfo{booktitle}{\emph{IJCAI}}.
  \bibinfo{pages}{2435--2441}.
\newblock


\bibitem[\protect\citeauthoryear{Zhou and He}{Zhou and He}{2017}]%
        {M2VW}
\bibfield{author}{\bibinfo{person}{Yao Zhou} {and} \bibinfo{person}{Jingrui
  He}.} \bibinfo{year}{2017}\natexlab{}.
\newblock \showarticletitle{A Randomized Approach for Crowdsourcing in the
  Presence of Multiple Views}. In \bibinfo{booktitle}{\emph{{IEEE} ICDM}}.
  \bibinfo{pages}{685--694}.
\newblock


\bibitem[\protect\citeauthoryear{Zhou, Ying, and He}{Zhou
  et~al\mbox{.}}{2017}]%
        {DBLP:conf/sdm/ZhouYH17}
\bibfield{author}{\bibinfo{person}{Yao Zhou}, \bibinfo{person}{Lei Ying}, {and}
  \bibinfo{person}{Jingrui He}.} \bibinfo{year}{2017}\natexlab{}.
\newblock \showarticletitle{MultiC\({}^{\mbox{2}}\): an Optimization Framework
  for Learning from Task and Worker Dual Heterogeneity}. In
  \bibinfo{booktitle}{\emph{SDM}}. \bibinfo{pages}{579--587}.
\newblock


\bibitem[\protect\citeauthoryear{Zhu}{Zhu}{2015}]%
        {DBLP:conf/aaai/Zhu15}
\bibfield{author}{\bibinfo{person}{Xiaojin Zhu}.}
  \bibinfo{year}{2015}\natexlab{}.
\newblock \showarticletitle{Machine Teaching: An Inverse Problem to Machine
  Learning and an Approach Toward Optimal Education}. In
  \bibinfo{booktitle}{\emph{AAAI}}. \bibinfo{pages}{4083--4087}.
\newblock


\bibitem[\protect\citeauthoryear{Zhu, Ghahramani, and Lafferty}{Zhu
  et~al\mbox{.}}{2003}]%
        {DBLP:conf/icml/ZhuGL03}
\bibfield{author}{\bibinfo{person}{Xiaojin Zhu}, \bibinfo{person}{Zoubin
  Ghahramani}, {and} \bibinfo{person}{John~D. Lafferty}.}
  \bibinfo{year}{2003}\natexlab{}.
\newblock \showarticletitle{Semi-Supervised Learning Using Gaussian Fields and
  Harmonic Functions}. In \bibinfo{booktitle}{\emph{ICML}}.
  \bibinfo{pages}{912--919}.
\newblock


\bibitem[\protect\citeauthoryear{Zhu, Singla, Zilles, and Rafferty}{Zhu
  et~al\mbox{.}}{2018}]%
        {MT_overview}
\bibfield{author}{\bibinfo{person}{Xiaojin Zhu}, \bibinfo{person}{Adish
  Singla}, \bibinfo{person}{Sandra Zilles}, {and} \bibinfo{person}{Anna
  Rafferty}.} \bibinfo{year}{2018}\natexlab{}.
\newblock \showarticletitle{{An Overview of Machine Teaching}}.
\newblock \bibinfo{journal}{\emph{ArXiv}} (\bibinfo{date}{Jan.}
  \bibinfo{year}{2018}).
\newblock
\showeprint[arxiv]{cs.LG/1801.05927}


\end{thebibliography}
\flushend

\end{document}